\documentclass{article}

\usepackage[preprint]{neurips_2025}

\usepackage[utf8]{inputenc} 
\usepackage[T1]{fontenc}    
\usepackage{hyperref}       
\usepackage{url}            
\usepackage{booktabs}       
\usepackage{amsfonts}       
\usepackage{nicefrac}       
\usepackage{microtype}      
\usepackage{xcolor}         
\usepackage{multirow}

\usepackage{enumitem}
\usepackage{amsmath, amsthm}
\usepackage[capitalise]{cleveref}
\usepackage{graphicx}
\usepackage{wrapfig} 
\newtheorem{proposition}{Proposition}
\newtheorem{definition}{Definition}
\newtheorem{lemma}{Lemma}
\usepackage{algorithm}
\usepackage{algorithmic}
\usepackage{subcaption}
\usepackage{siunitx}
\usepackage{tabularx}

\title{MatRL: Provably Generalizable Iterative Algorithm Discovery via Monte-Carlo Tree Search}

\author{%
  Sungyoon Kim \\
  Stanford University\\
  \texttt{sykim777@stanford.edu} \\
  \And
  Rajat Vadiraj Dwaraknath \\
  Stanford University\\
  \texttt{rajatvd@stanford.edu} \\
  \AND
  Longling geng \\
  Stanford University\\
  \texttt{gll2027@stanford.edu} \\
  \And
  Mert Pilanci \\
  Stanford University\\
  \texttt{pilanci@stanford.edu} \\
}

\begin{document}

\maketitle

\begin{abstract}

Iterative methods for computing matrix functions have been extensively studied and their convergence speed can be significantly improved with the right tuning of parameters and by mixing different iteration types \cite{higham1990fast}.
Hand‑tuning the design options for optimal performance can be cumbersome, especially in modern computing environments: numerous different classical iterations and their variants exist, each with non‑trivial per‑step cost and tuning parameters. To this end, we propose \textbf{MatRL} -- a reinforcement learning based framework that automatically discovers iterative algorithms for computing matrix functions. The key idea is to treat algorithm design as a sequential decision-making process. Monte-Carlo tree search is then used to plan a \emph{hybrid} sequence of matrix iterations and step sizes, tailored to a specific input matrix distribution and computing environment. Moreover, we also show that the learned algorithms provably generalize to sufficiently large matrices drawn from the same distribution. Finally, we corroborate our theoretical results with numerical experiments demonstrating that MatRL produces algorithms that outperform various baselines in the literature.




\end{abstract}

\section{Introduction}

Matrix functions are everywhere - ranging from classical applications in control theory \cite{denman1976matrix}, high-dimensional ODEs \cite{higham2008functions}, theoretical particle physics \cite{chen2014newton}, \cite{lin2009fast}, Markov models \cite{waugh1967fractional}, to some recent applications in machine learning, e.g. covariance pooling \cite{li2018towards}, \cite{wang2020deep}, \cite{song2021approximate}, graph signal processing \cite{defferrard2016convolutional}, \cite{maskey2023fractional}, contrastive learning \cite{richemond2023edge}, and optimizer design \cite{gupta2018shampoo}, \cite{jordan2024muon}, \cite{ahn2025dion}. For more applications see \cite{higham2008functions} and references therein. It isn’t a surprise that computing matrix functions in a fast, precise, and stable manner has received numerous attention and many have tried to develop fast algorithms with guarantees that it will work in some sense.


Iterative algorithms to compute matrix functions are particularly attractive in modern applications as they can avoid computing the matrix function directly using the singular value decomposition (SVD). Further, termination criteria can be chosen based on the needs of the application, which can lead to faster and more stable algorithms. Additionally, these algorithms are differentiable, leading to potential applications in auto-differentiation based settings \cite{song2022fast}. From an algorithm design perspective, it would be ideal to find an algorithm that is computationally efficient, numerically stable and has faster convergence guarantees. One can imagine using different iterations from different algorithms at each step (e.g. mixing the algorithms) and specifically tuning the parameters jointly to accelerate the algorithm. However, due to the large search space (see \cref{sec:cmf} and references therein to see a vast array of existing iterative algorithms) and highly sensitive and non-trivial per iteration costs in modern computing environments, handcrafting the ideal algorithm is tedious and impractical.  

\sisetup{
  separate-uncertainty=true,
  table-align-uncertainty=true,
  uncertainty-separator = {\,$\pm$\,},
}

\begin{table}
\caption{Computation times for matrix inverse and multiplication ($4096\times4096$) on CPU and GPU in single (float32) and double (float64) precision. The right‐most column shows the ratio of inverse time to matmul time. All computations done using PyTorch on an RTX A6000 GPU.}
\centering
\label{table:cpugputimes}
\begin{tabular}{
  ll
  S[table-format=3.1(2)]
  S[table-format=3.1(2)]
  S[table-format=1.2]
}
\toprule
Device & Precision & {Inverse (ms)} & {MatMul (ms)} & {Inv : MatMul ratio} \\
\midrule
CPU & float32 & 251.2 \pm 8.7  &  48.8 \pm 1.8  & 5.15 \\
    & float64 & 356.6 \pm 16.3 & 122.9 \pm 4.6  & 2.90 \\
GPU & float32 &  28.3 \pm 0.5  &   5.6 \pm 0.1  & 5.05 \\
    & float64 & 423.9 \pm 6.9  & 252.6 \pm 6.4  & 1.68 \\
\bottomrule
\end{tabular}
\end{table}
A motivating experiment is presented in \Cref{table:cpugputimes}. In single precision, the cost of computing a $4096\times4096$ inverse is roughly five times that of a matrix multiply on both CPU and GPU. However, switching to double precision more than halves this ratio -- dropping to about 2.9 on CPU and 1.7 on GPU -- because matrix multiplies become comparatively more expensive. This pronounced change in relative costs with precision underscores the need for automatic algorithm discovery, as the optimal sequence of matrix iterations will depend sensitively on both hardware and numerical precision. Besides the change in relative speeds, the gains from switching to GPU turns into a slowdown with double precision for both inverse and matmuls. This is due to the significantly lower emphasis on high precision compute capability in modern GPUs. 


To this end, we propose an automated solution based on Monte-Carlo tree search to decide which combination of iterations and parameters one should use given a desiderata of the user. Our solution assumes that the matrix of interest is sampled repeatedly from a certain symmetric random matrix distribution, and we want to find a good algorithm for that matrix distribution. The main idea is that iterative algorithms can be understood as a sequential decision-making process: at each step, one should choose which iteration to use and which parameters to use in that iteration. This corresponds to choosing actions in decision making, where the choice leads to the next step. At each step we get a certain reward signal based on the given desiderata, so that we can evaluate whether the action was worth it or not. Finding a good algorithm mounts to finding a good planning strategy for the given environment. Specifically,

\begin{itemize}[left=0pt]
    \item We propose \textbf{MatRL} (\cref{MatRL}), an automated algorithm searching scheme based on Monte-Carlo tree search. 
    \item The algorithms we find are faster than existing baselines, and even faster than implementations in torch.linalg. Moreover, the found algorithms differ with problem sizes, computation environment, and precision (\cref{sec:Discovered}), meaning that MatRL adapts to different enviroments with ease.
    \item We have a guarantee using random matrix theory that the found algorithm will generalize to different matrices drawn from the same distribution (\cref{sec:RMT}), and matrices drawn from the distribution with identical limiting eigenvalues. 
\end{itemize}

The paper is organized as the following: in \cref{sec:Background} we discuss relevant background on iterative matrix function computation algorithms, Monte-Carlo tree search and learning algorithms via RL. Next we describe our environment in \cref{sec:Environment} by describing how the states, actions, state transition, and reward signals are defined. In \cref{sec:RMT}, we provide the generalization guarantee stemming from limiting distribution of the spectrum. We show experimental results in \cref{sec:exp} showing the performance and adaptivity of MatRL, and conclude the paper in \cref{sec:conclusion}.

\section{Background}
\label{sec:Background}
\subsection{Iterative Methods to Compute Matrix Functions}
\label{sec:cmf}
The basic idea of obtaining an iteration to compute matrix functions is using Newton's method \cite{higham2008functions}. For instance, say we want to compute the matrix square root. We would like to compute the root of the function $f(X) = X^2 - A$, hence Newton's method can be written as
\begin{equation}
\label{sqrt_newton}
X_{k+1} = X_k - f'(X_k)^{-1}f(X_k) = X_k  - \frac{1}{2}X_k^{-1}(X_k^2 - A) = \frac{1}{2}(X_k + X_k^{-1}A).
\end{equation}
Newton's method uses a first-order approximation of $f$ at $X_k$: we may use higher order approximations to obtain Chebyshev method \cite{li2011chebyshev} or Halley's method \cite{nakatsukasa2010optimizing}, \cite{guo2010newton}. These higher-order methods converge cubically with appropriate initialization, but each iteration is slower than Newton's method. We could also use inverse-free methods such as Newton-Schulz and its variants \cite{higham2008functions}, \cite{higham1997stable}, which approximates $X_k^{-1}$ with a polynomial of $X_k$.

Appropriate scaling and shifting of the spectrum to yield faster convergence has been a popular idea \cite{nakatsukasa2010optimizing}, \cite{byers2008new}, \cite{byers1987solving}, \cite{chen2014newton}, \cite{iannazzo2003note}, \cite{hoskins1979faster}, \cite{pan1991improved}. The intuition is that by introducing additional parameters and solving an optimization problem on the spectrum, we can find a sequence of optimized parameters that depend on the spectrum of $A$, the matrix that we would like to compute matrix function. For instance, \cite{chen2014newton} finds a cubic function $h : [0,1] \rightarrow [0,1]$ that maximizes $h'(0)$ to find better scaling of Newton-Schulz iteration. One drawback of these approaches is that in many cases we need to compute smallest / largest eigenvalues of the matrix \cite{byers2008new}, \cite{chen2014newton}, \cite{pan1991improved}, \cite{nakatsukasa2010optimizing}, \cite{hoskins1979faster}, which may be expensive to compute. In this case we use approximations of the smallest eigenvalue, such as $1/\lVert A ^{-1} \rVert_F$. Another drawback is that each new scaling scheme needs a complicated mathematical derivation and proof.

Naively applying Newton's method can be unstable for computing matrix square-root and $p$-th root. To ensure stability, 
we can introduce an auxillary variable $Y_k$ and simultaneously update $X_k$ and $Y_k$ \cite{higham1997stable}, \cite{iannazzo2006newton}. We will refer to such iterations as coupled iterations. Coupled iterations are obtained by manipulating the formula so that we do not have $A$ in the iteration. Going back to Newton's method for computing square roots in \cref{sqrt_newton}, we can introduce an auxillary variable $Y_k = A^{-1}X_k$ and initialize $X_0 = A, Y_0 = I$ to obtain the coupled iteration known as Denman-Beavers iteration \cite{denman1976matrix},
\begin{equation}
\label{eq:sqrt_coupled}
\begin{cases}
X_{k+1} = \frac{1}{2}(X_k + Y_k^{-1})\\
Y_{k+1} = \frac{1}{2}(Y_k + X_k^{-1}).
\end{cases}    
\end{equation}
Using perturbation analysis, \cite{higham1997stable} shows that the iteration in \cref{sqrt_newton} is unstable when the condition number $\kappa(A) > 9$, whereas the iteration in \cref{eq:sqrt_coupled} is stable.  

Iterative algorithms are not limited to Newton's method. We may have fixed-point iterations such as Visser iteration \cite{higham2008functions}, where we iteratively compute
\[
X_{k+1} = X_k + \alpha_k (A - X_k^2),
\]
to compute matrix square root. We may also use higher-order rational approximations of the function of interest to obtain algorithms that converge in only a few steps \cite{nakatsukasa2016computing}, \cite{gawlik2019zolotarev}, and with sufficient parallelization they can be faster than existing methods. 

\subsection{Monte-Carlo Tree Search}
\label{sec:MCTS}

Suppose we have a deterministic environment $\mathcal{E}$ which is defined by a 5-tuple, $(\mathcal{S}, \mathcal{A}, \mathcal{T}, r, t)$. $\mathcal{S}$ denotes the set of states, $\mathcal{A}$ denote the set of actions that one can take in each state, $\mathcal{T}: \mathcal{S} \times \mathcal{A} \rightarrow \mathcal{S}$ gives the how state transition occurs from state $s$ when we apply action $a$, $r: \mathcal{S} \times \mathcal{A} \rightarrow \mathbb{R}$ gives the amount of reward one gets when we do action $a$ at state $s$, and $t: \mathcal{S} \rightarrow \{0, 1\}$ denotes whether the state is terminal or not. Monte-carlo tree search enables us to find the optimal policy \cite{browne2012survey} $\pi: \mathcal{S} \rightarrow \mathcal{A}$ that decides which action to take at each state to maximize the reward over the trajectory that $\pi$ generates. 

The basic idea is to traverse over the ``search tree" in an asymmetric manner, using the current value estimation of each state. Each node of the search tree has a corresponding state $s$, its value estimation $V_s$, and visit count $N_s$. The algorithm is consisted of four steps, the selection step, the expansion step, the simulation step, and the backpropagation step.

\textbf{During the selection step,} the algorithm starts at the root node $s_0$ and selects the next node to visit using $V_{\mathcal{T}(s,a)}$ and $N_{\mathcal{T}(s,a)}$. One widely-used method is using the upper confidence bounds for trees (UCT) \cite{kocsis2006bandit}. At node $s$, the UCT is defined by
\[
V_{\mathcal{T}(s,a)} + C \sqrt{\frac{\log(N_{s})}{N_{\mathcal{T}(s,a)}}},
\]
where $C$ determines the exploration-exploitation tradeoff. The algorithm selects the next state $\mathcal{T}(s,a)$ that maximizes UCT, and such selection continues until the algorithm meets a node of which not all child nodes have been visited. 

\textbf{In the expansion step,} the algorithm adds a child node that has not been visited to the search tree.

\textbf{In the simulation step,} the algorithm starts from the added node in the expansion step and uses a default policy to generate a trajectory until they meet a termination criterion. The default policy can either be a random policy or handcrafted heuristics \cite{chaslot2008progressive}. 

\textbf{Finally, in the backpropagation step, } each value estimation $V_s$ on the trajectory is updated.

Monte-Carlo tree search gained its popularity to obtain strategies for games such as Tictactoe \cite{veness2011monte} or Go \cite{silver2017mastering}, \cite{hoock2010intelligent}, as well as real-time stategic games \cite{soemers2014tactical}. Not only that, the algorithm was also applied to combinatorial optimization problems \cite{sabharwal2012guiding}, \cite{rimmel2011optimization}, symbolic regression \cite{kamienny2023deep}, and complex scheduling problems \cite{chaslot2006monte}, \cite{matsumoto2010evaluation}, \cite{li2021effective} - which is most relevant to our work.

\subsection{Automated Algorithm Discovery}

We are not the first to discover algorithms using ideas from sequential decision-making. RL-based approaches which parametrize the policy as a neural network have proven to be successful: \cite{li2017learning} learns to optimize neural networks using an RL framework. In their framework, the states consist of past variables $x_i$, past gradients $\nabla f(x_i)$, and past objectives $f(x_i)$, and the policy aims to learn appropriate $\Delta x$. \cite{fawzi2022discovering} used reinforcement learning to find faster matrix multiplication algorithms, and \cite{mankowitz2023faster} finds faster sorting algorithms. \cite{khodak2024learning} has a similar flavor with our work, where they use contextual bandits to optimize relaxation parameters in symmetric success-over-relaxation.

\section{MatRL: Iterative Matrix Function Algorithm Search via RL}
\label{sec:Environment}

\subsection{Objective}
\label{sec:Objective}

Let $A \sim \mathcal{D}$, where $\mathcal{D}$ is a symmetric random matrix distribution defined in $\mathbb{R}^{n \times n}$ and $f:\mathbb{R} \rightarrow \mathbb{R}$ is the function that we would like to compute. We follow the definition of matrix functions in \cite{higham2008functions}. For a symmetric matrix $A$, when we diagonalize $A = UDU^{T}$ for an orthogonal matrix $U$, $f(A) = Uf(D)U^{T}$ where $f(D)$ applies $f$ to the diagonal entries of $D$. Denote $f_1(X,Y,A,a_1), f_2(X,Y,A,a_2), \cdots f_m(X,Y,A,a_m)$ as $m$ choices of iterations that we can use, $a_j \in \mathbb{R}^{n_j}$ as tunable parameters for each $f_j$, and $T_j$ the wall-clock time needed to run $f_j$. We use two variables $(X, Y)$ as input to accomodate coupled iterations, and $Y$ is not used for iterations that are not coupled. For instance, for $f=\sqrt{\cdot}$, $f_1(X, Y, a_1)$ can be the scaled Denman-Beavers iteration
\[
f_1(X, Y, A, a_1) = \Big(\frac{1}{2} (a_{11}X + (a_{12}Y)^{-1}), \frac{1}{2} (a_{12}Y + (a_{11}X)^{-1}) \Big),
\]
whereas $f_2(X, Y, A, a_2)$ can be the scaled Visser iteration
\begin{equation}
\label{eq:Visser_naive}
f_2(X, Y, a_2) = \Big(a_{21}X + a_{22}(A - X^2), Y \Big).
\end{equation}
Here, $n_j$ denotes the number of tunable parameters in iteration $f_j$. Also, assume the error tolerance $\epsilon_{\text{tol}}$ is given.

Now, we specify the class of matrix iterations and the custom loss function $\mathcal{L}: \mathbb{R}^{n\times n} \times \mathbb{R}^{n\times n} \rightarrow \mathbb{R}$ that determines the termination condition. For this we define congruence invariant matrix functions.
\begin{definition}
(Congruence Invariant Diagonal Preserving) Let $f: (\mathbb{R}^{n \times n})^k \rightarrow \mathbb{R}^{m\times m}$ a matrix function that takes $k$ matrices as input and outputs a matrix. If $f(QX_1Q^{T}, QX_2Q^{T}, \cdots QX_kQ^{T}) = Qf(X_1, X_2, \cdots X_k)Q^{T}$ for all orthogonal $Q$, $f$ is congruent invariant. If $f(X_1, X_2, \cdots, X_k)$ is diagonal for diagonal $X_1, X_2, \cdots X_k$, $f$ is diagonal preserving. \\
For functions that take matrix tuple as an input and outputs a matrix tuple, $F: (\mathbb{R}^{n \times n})^k \rightarrow (\mathbb{R}^{m\times m})^{l}$, we call $F$ congruent invariant diagonal preserving if $F = (f_1, f_2, \cdots f_l)$ and all $f_i$ are congruent invariant and diagonal preserving for $i \in [l]$.
\end{definition}
We limit the matrix iterations and the loss function to be congruent invariant functions, regarding $A$ also as an input. For instance, the Newton-Schulz iteration to compute matrix inverse \cite{pan1991improved} $X_{k+1} = 2X_k - X_kAX_k,$
$f(X,A) = 2X - XAX$ is a congruence invariant function and the iteration falls into our framework. Such limitation will enable us to understand actions and losses as functions of the spectrum of $X, Y$, and $A$, and as we will see in the next section, it will enable us to see the whole environment only as a function of the spectrum. 

Our objective is given a sample $A$ from $\mathcal{D}$, finding a sequence of iterations and coefficients $f_{t_1}(\cdot, a_{t_1}), f_{t_2}(\cdot, a_{t_2}), \cdots f_{t_N}(\cdot, a_{t_N})$ that is a solution to
\begin{equation}
\label{eq:obj}
  \max_{N, t_i \in [m], a_{t_i} \in \mathbb{R}^{n_i},\ i \in [N]} \quad -\sum_{i=1}^{N} T_{t_i} \quad \text{subject \ to} \quad  \mathcal{L}(X_{N+1}, A) \leq \epsilon_{\text{tol}},
\end{equation}
where $X_0 = A, Y_0 = I$ or $X_0 = I, Y_0 = A$ depending on $f$, $(X_{k+1},Y_{k+1}) = f_{t_k}(X_k, Y_k,  a_{t_k})$ and $N$ is also optimized. Note that similar ideas have also appeared in optimal control \cite{evans1983introduction} in the context of minimum-time control problems.

The optimal solution of \cref{eq:obj} naturally corresponds to finding the optimal iterative algorithm 
 \[
 X_{k+1}, Y_{k+1} \leftarrow f_{t_k}(X_k, Y_k, a_{t_k}), \quad k \in [N].
 \]
 that arrives as $\mathcal{L}(X,A) \leq \epsilon_{\text{tol}}$ as fast as it can. 
\subsection{The Environment}
\label{sec:env}

Here we elaborate how we formulate the problem in \cref{eq:obj} to a sequential decision-making problem by describing $\mathcal{E} = (\mathcal{S}, \mathcal{A}, \mathcal{T}, r, t)$, the 5-tuple specified in \cref{sec:MCTS}.

The simplest way to define the state and action variables is by setting each state as a tuple $(X, Y)$ that corresponds to the matrices $(X_k, Y_k)$, and setting each action as a $n_j + 1$ - tuple $(j, k_1, k_2, \cdots k_{n_j})$, where $j \in [m]$ denotes the iteration $f_j$ and $k_1, k_2, \cdots k_{n_j}$ denotes the parameters for $f_j$. The state transition $\mathcal{T}$ simply becomes  
\[
\mathcal{T}(X, Y, j, k_1, k_2, \cdots k_{n_j}) = f_j(X, Y, k).\ 
\]
At each transition, we get rewarded by $-T_j$, the negative time needed to run the iteration. The terminal state is when $\mathcal{L}(X, A) \leq \epsilon_{\text{tol}}$. Our environment stems from this basic formulation, but we have important implementation details that we elaborate below.

\textbf{Spectrum as state variables} Having matrices each state can consume a lot of memory, and state transition may be slow. Instead, we use $(s_1, s_2)$, where $s_1, s_2 \in \mathbb{R}^{n}$ corresponds to the eigenvalues of $X, Y$. Such parametrization is justified by the fact that both the transition $f_k(X, Y, A, a) = (X', Y')$ and the termination criteria $\mathcal{L}$ can be expressed only using the spectrum for our iterations of interest. 

The core intuition is that when we write $A = UD_AU^{T}$, if $X = UD_XU^{T}$ and $Y = UD_YU^{T}$ for the same $U$ and diagonal $D_A, D_X, D_Y$, the next states $X', Y'$ are similar to $A$. Moreover, we can see that the spectrum of $X', Y'$ only depends on the spectrum of $X, Y$ and $A$. Induction finishes the proof. The similarity result and the congruent invariance of $\mathcal{L}$ shows that the termination criteria only depends on the spectrurm. We defer the proof to Appendix A.


\textbf{Decoupled actions} 
We have $n_j+1$ tuple of possible actions each state. As they are continuous variables, the search space becomes huge. To mitigate this, we decouple the state transition $\mathcal{T}$ into $n_j+1$ stages: at stage 1, 2, $\cdots n_j$, only parameters $k_1, \cdots k_{n_j-1}$ are chosen. At state $n_j+1$, we get rewarded by $-T_j$, choose next iteration, and state transition happens. 

\textbf{Dealing with coupled iterations}
For computing matrix roots, some iterations are coupled (e.g., Denman-Beavers), while others like the scaled Visser iteration are not, making it challenging to mix them directly. Applying uncoupled iterations alone can break essential relationships (e.g., $Y_k X_k^{-1} = A$ for Denman-Beavers), potentially leading to incorrect results. To address this, we propose either augmenting uncoupled steps with coupled ones or tracking a boolean flag IsCoupled that governs when and how to restore consistency between variables before proceeding.
\subsection{The Searching Strategy}
\label{sec:SearchingStrategy}

Here we describe the details of Monte-Carlo tree search in the predescribed environment $\mathcal{E}$. Let's note $n, v: \mathcal{S} \rightarrow \mathbb{N}$ the number of child nodes and visit count of that node, respectively. If all parameters $j, k_1, \cdots k_{n_j}$ are chosen and the state is ready for state transition, we call the state \emph{transitionable}. We use progressive widening \cite{couetoux2011continuous} to deal with the continuous parameter space: during the selection stage, if the node is transitionable and hasn't visited all possible children, or the number of child nodes $n(s) \leq Cv(s)^{\alpha}$ for some hyperparameters $C, \alpha$, we go to the expansion step. If not, we choose the next node with UCT \cite{kocsis2006bandit}. In the expansion stage, we add a child node. Choosing the iteration is discrete and we choose them depending on IsCoupled flag. To choose the parameters, we first sample randomly for $E$ steps, then jitter around the best parameter found. In the rollout stage, we have predetermined baseline algorithms and run one of them to estimate the value of the state. At last, we backpropagate by using the Bellman equation
\[
V_s = \max_{a \in \mathcal{A}} V_{\mathcal{T}(s,a)} + r(s,a).
\]
Our search method is summarized in \cref{MatRL}. Details on the parameters for each experiment and a thorough description of \cref{MatRL} can be found in Appendix C.
\begin{algorithm}
\caption{MatRL: Monte-Carlo Tree Search for Algorithm Discovery}
\begin{algorithmic}
  \label{MatRL}   
  \STATE \textbf{Input:} $C, \alpha, \epsilon_{\text{tol}}, E, T, RolloutList, \mathcal{L}, A \sim \mathcal{D}$
  \STATE Initialize $c, n, t, cp \gets 0$, $V \gets -\text{INF}$, $cp[s_{root}] \gets 1$, bestpath, bestrollout $\gets 0$ // Each correspond to number of children, visit count, Transitionable, IsCoupled, and value estimation.
  \FOR{$i = 1$ to $T$}
  \STATE $s \gets s_{root}$
  \WHILE {$Expandable(s)$ == False and $\mathcal{L}(s) \leq \epsilon_{\text{tol}}$}
        \STATE // Expandable if $s$ is transitionable and has a child node yet visited, or $c(s) \leq Cn(s)^{\alpha}$
        \STATE $s \gets Best_{UCB}(s)$  
  \ENDWHILE
  \STATE $s \gets ExpandNode(s)$ // Here we expand after we look at $cp(s)$
  \STATE $r \gets SampleRolloutList()$ // Here we sample from RolloutList, the baselines selected for rollout. If the baseline is coupled but $cp(s)= False$, we attach an additional coupling step at the front 
  \STATE $s \gets r(s)$
  \STATE bestpath, bestrollout $\gets backpropagate(s)$ // Here we use bellman equation. If $V(s_{root})$ was updated, update bestpath and bestrollout
  \ENDFOR
  \RETURN bestpath $\oplus$ bestrollout
\end{algorithmic}
\end{algorithm}
\vspace{-0.5cm}
\section{Random Matrices and Generalization Guarantees}
\label{sec:RMT}
The objective in \cref{eq:obj} aims to find the optimal algorithm for a given matrix $A$ sampled from $\mathcal{D}$. Here we show that under certain assumptions, the found iterative algorithm has a sense of generalization capability to a different matrix distribution $\mathcal{D}'$ with the same limiting distribution. Two main ideas for the proof is: first, the loss curve only depends on the spectrum. Second, if the limiting distributions are identical for $\mathcal{D}$ and $\mathcal{D}'$, the sampled matrices' spectrum will be similar if the matrix is sufficiently large. Hence, the loss curve will be similar for two matrices $X \sim \mathcal{D}$, $Y \sim \mathcal{D}'$ for sufficiently large matrices, and if the algorithm works well for $X$, it will work well for $Y$. The specific guarantee that we have is in \cref{prop:guarantee}. The proof is deferred to Appendix A.

\begin{proposition}
\label{prop:guarantee}
(Generalization of the discovered algorithm) Say we have a sequence of symmetric random matrix distributions $\mathcal{P}_m, \mathcal{Q}_m$ defined in $\mathbb{R}^{m \times m}$, and denote random matrices sampled from $\mathcal{P}_m$, $\mathcal{Q}_m$ as $X, Y$. Let the empirical eigenvalue value distribution of $X \sim \mathcal{P}_m, Y \sim \mathcal{Q}_m$ be $\mu_m(X), 
 \nu_m(Y)$, and their support be $S_m, S'_m$, respectively. Now, suppose\\
(i) (Identical limiting distribution) 
\[
\mathbb{P}(\mu_m(X) \Rightarrow \mu^{*}) = \mathbb{P}(\nu_m(Y) \Rightarrow \mu^{*}) = 1,
\]
i.e. both $\mu_m(X)$ and $\nu_m(Y)$ converges weakly to a common distribution $\mu^{*}$ with probability 1.\\
(ii) (Interval support of the limiting distribution) The support of $\mu^{*}$ is an interval $[a, b]$. \\
(iii) (Convergence of support) We have
\[
\lim_{m \rightarrow \infty} \mathbb{P}(S_m \subseteq [a-\epsilon, b+\epsilon]) = \lim_{m \rightarrow \infty} \mathbb{P}(S'_m \subseteq [a-\epsilon, b+\epsilon]) = 1,
\]
for all $\epsilon > 0$.\\
With the assumptions above, let $f$ be the matrix function we would like to compute, $f^{*}_k$ the step $k$ transformation of eigenvalues of the algorithm found by \cref{MatRL}, and $\mathcal{L}$ be the loss. Assume $f, f_k^{*}, \mathcal{L}$ are continuous in $[a - \epsilon_0,b + \epsilon_0]$ for some $\epsilon_0 > 0$. Write the empirical loss of the random matrix $X$ as
\[
\mathcal{L}_k(X) = \frac{1}{m} \sum_{i=1}^{m} L(f(\lambda_i), f_k^{*}(\lambda_i)),
\]
where $\lambda_i$ are eigenvalues of $X$.\\
Then, there exists $M_{\epsilon, \delta}$ such that
\[
m \geq M_{\epsilon, \delta} \Rightarrow \mathbb{P}_{X \sim \mathcal{P}_m, Y \sim \mathcal{Q}_m} [|\mathcal{L}_k(X) - \mathcal{L}_k(Y)| < \epsilon] \geq 1 - \delta.
\]
\end{proposition}
Essentially, \cref{prop:guarantee} states that if we find an algorithm $\{f^{*}_k(X,Y,a_k)\}_{k=1}^{N}$ using \cref{MatRL} and it works well for a certain matrix $A$, as the loss value $\mathcal{L}_k(X)$ and $\mathcal{L}_k(Y)$ are similar for sufficiently large $m$, it will work well for any matrix within the same distribution, and also generalize to distributions with the same limiting distribution if $m$ is sufficiently large.

\section{Experiments}
\label{sec:exp}

\subsection{Discovered Algorithms}
\label{sec:Discovered}

\textbf{Different matrix functions} Here we show loss curves of two different matrix functions, $\text{sign}(A)$ and $A^{1/2}$. Results for inverse and $A^{1/3}$ can be found in Appendix E. Wishart denotes $A = \frac{X^\top X}{3d} + \epsilon_{\text{stb}} I$ where $X \in \mathbb{R}^{d/4 \times d}$, $X_{ij} \sim \mathcal{N}(0, 1)$ i.i.d.. $\epsilon_{\text{stb}} = 1e-3$ exists for numerical stability. "Hessian of Quartic" is the indefinite Hessian of a $d$-dimensional quartic $\sum_i z_i^4/4 - z_i^2/4$ evaluated at a random point $z \sim \mathcal{N}(0, \mathbf{I_d})$. The loss function for matrix sign was $\lVert X - \text{sign}(A) \rVert_F / \sqrt{d}$, and the loss function for matrix square root was $\lVert X^2 - A \rVert_F / \lVert A \rVert_F$. For a list of baselines we compare with, see \cref{table:listbaselines} and references therein. A detailed explanation of each baseline can be found in Appendix B.
\begin{table}[H]
\centering
\caption{List of baselines for matrix sign and square root}
\label{table:listbaselines}
\begin{tabular}{
  ll
  S[table-format=3.1(2)]
  S[table-format=3.1(2)]
  S[table-format=1.2]
}
\toprule
Matrix function & List of baselines \\
\midrule
Sign & Newton \cite{higham2008functions}, NS \cite{schulz1933iterative}, ScaledNewton \cite{byers2008new}, ScaledNS \cite{chen2014newton}, Halley\cite{nakatsukasa2010optimizing}\\
Square root& DB\cite{denman1976matrix}, NSV\cite{higham1997stable}(2.6), Scaled DB\cite{higham1997stable}, Visser\cite{higham2008functions}, Newton\cite{hoskins1979faster}\\
\bottomrule
\end{tabular}
\end{table}

The time vs loss curves for different matrix functions can be found in \cref{fig:diffmatfunc}. One thing to notice is that for matrix sign function, the algorithm found by MatRL is strictly better than torch.linalg.eigh, and for matrix square root the algorithm returns a nearly-similar error matrix within much smaller wall clock time. Depending on applications where exact matrix function is not neceessary, using algorithms found by MatRL can be an appealing choice. 
\begin{figure}[h]
\centering
    \begin{subfigure}[t]{0.47\textwidth}
        \includegraphics[width=\linewidth]{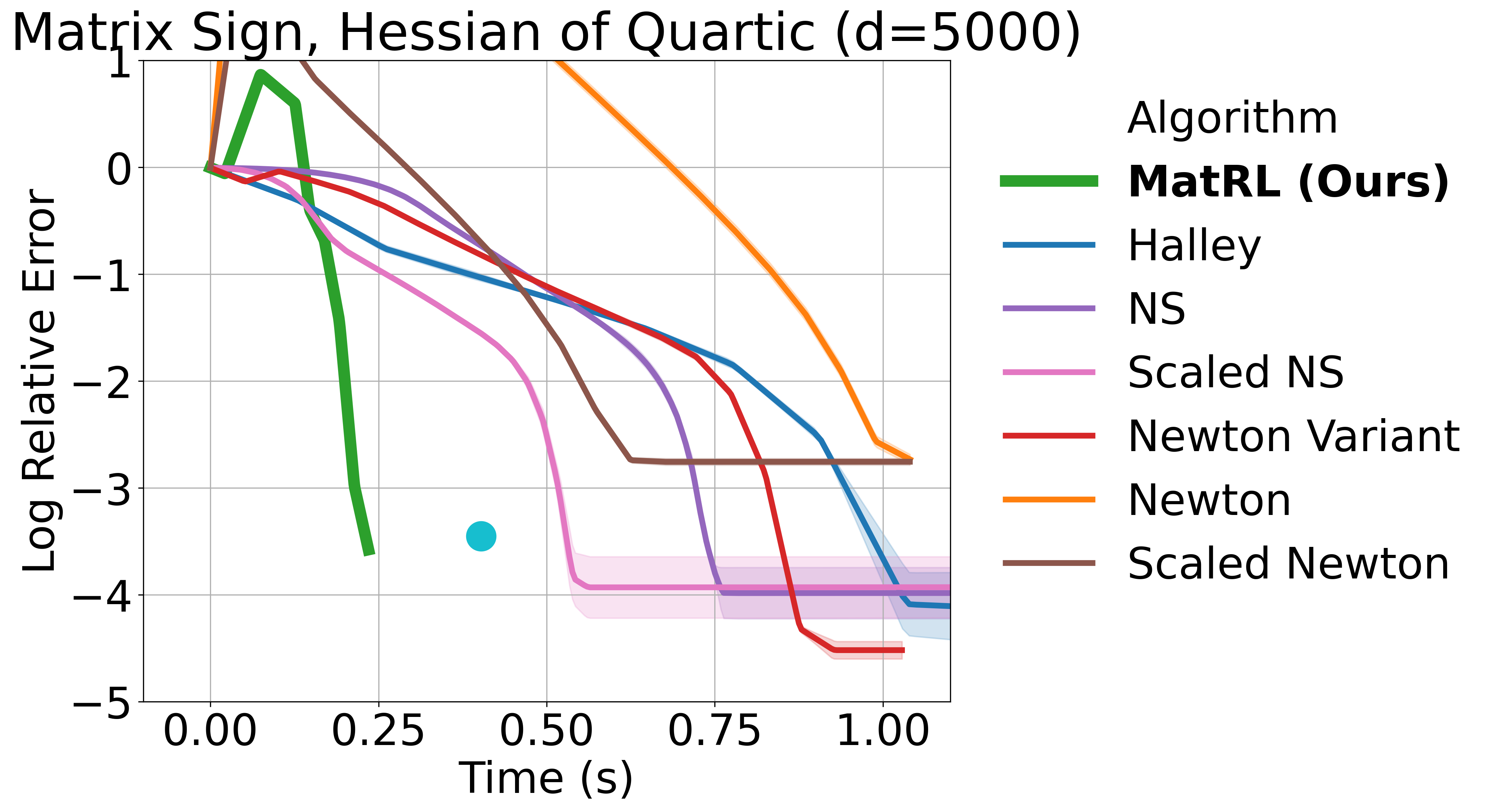}
        \label{MatRL:sign}
    \end{subfigure}
    \hfill
    \begin{subfigure}[t]{0.47\textwidth}
        \includegraphics[width=\linewidth]{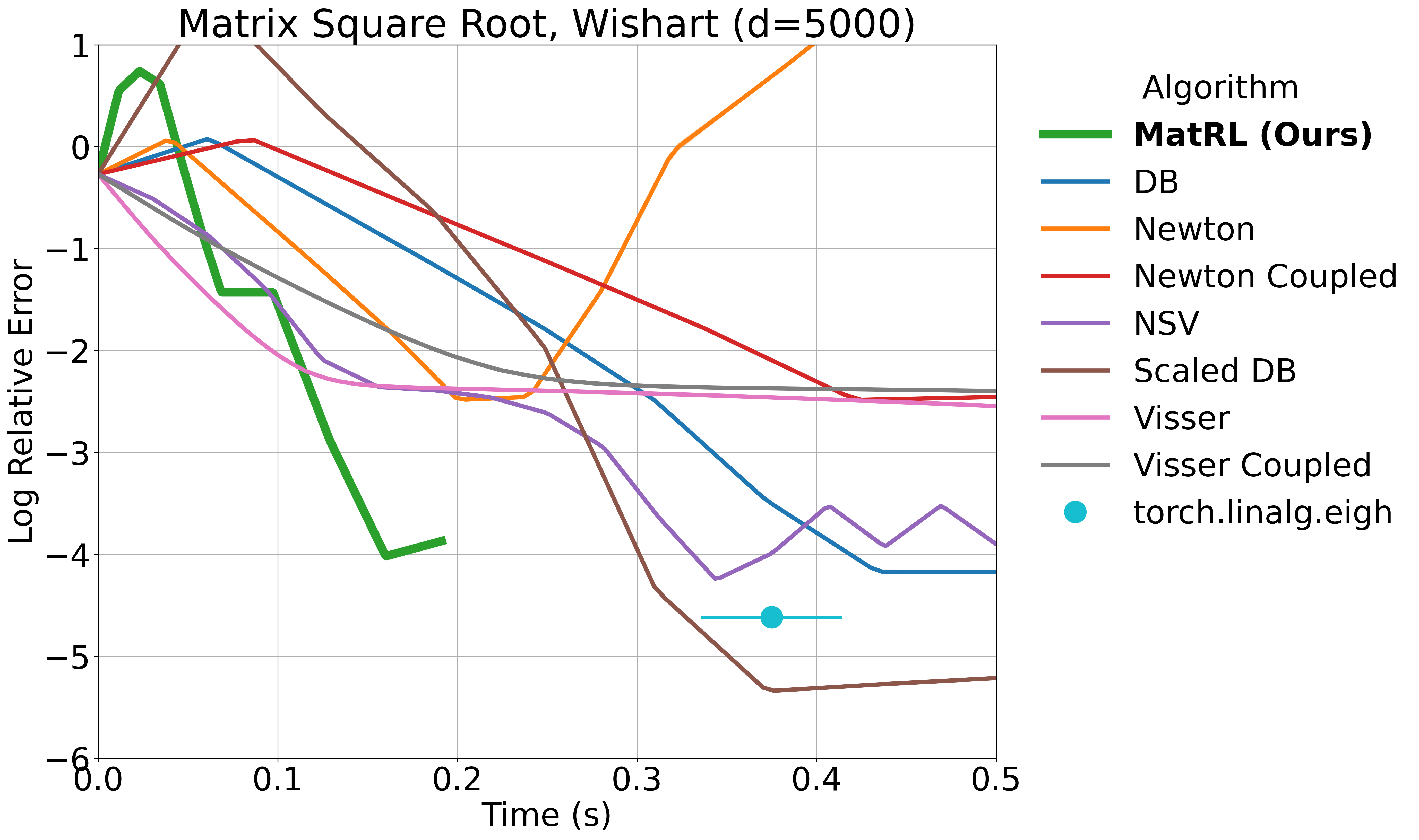}
        \label{MatRL:sqrt}
    \end{subfigure}
    \caption{Time versus loss curve for matrix sign and square root. Each iterative algorithms are plotted as baselines, and the wall-clock time to run torch.linalg.eigh is also reported. For matrix sign we have a clear benefit. For matrix square root the found algorithm sacrifices accuracy for time.}
    \label{fig:diffmatfunc}
\end{figure}
The algorithm in \cref{fig:diffmatfunc}, right is described in \cref{Foundalg-sqrt}. Here we see a tendancy that during the early iterations, the algorithm prefers Visser iteration, whereas for latter steps the algorithm prefers coupled Newton-Schulz iteration. An intuition for such mixed iteration is as follows: when we see the loss curve, we can notice that the Visser iteration converges fast at first, but it quickly stabilizes and becomes very slow. Hence MatRL learns to use the cheap iterations at first to find a good initial start for NewtonSchulz, and run NewtonSchulz from that initial parameter for faster convergence. 
\begin{algorithm}
\caption{Iterative SQRT for Wishart on GPU with $d = 5000$}
\begin{algorithmic}
  \label{Foundalg-sqrt}   
  \STATE \textbf{Input:} $A$
  \STATE Initialize $X_0 = A, Y_0 = I$
  \STATE Set $a \gets [2.092, 1.891, 3.714, 1.063, 0.920, 1], b \gets [1.983, 0.586, 2.385, 0.452, 0.900, 0.946]$, 
  \STATE // rounded off to three digits
  \FOR {$i=1$ to 6}
  \STATE $X_i = a_{i-1}X_{i-1}+b_{i-1}(A - X_{i-1}^2)$
  \ENDFOR
  \STATE $Y_6 = X_6A^{-1}$ 
  \FOR {$i=7$ to 9}
  \STATE $X_i = 0.5(3I -X_{i-1}Y_{i-1})X_{i-1}$
  \STATE $Y_i = 0.5Y_{i-1}(3I - X_{i-1}Y_{i-1})$
  \ENDFOR
  \RETURN $X_9$
\end{algorithmic}
\end{algorithm}

\textbf{Different setups yields different algorithms}
The mixing tendency appears for different matrix functions as well. Preference for a certain iteration over another emerges from two axes, the time it takes for an iteration and how effectively the iteration decreases the loss. Recalling the motivating example in \cref{table:cpugputimes}, different computing environments, e.g. hardware (CPU/GPU), precision (Single/Double), or even the size of the matrix can decide the best algorithm. Here we only demonstrate how precision can change optimal algorithms. A full list of different setups and found algorithms are in Appendix D. 
\begin{algorithm}[H]
\caption{SIGN on GPU ($d=5000$, \texttt{FLOAT})}
\label{SIGN_FLOAT}
\begin{algorithmic}
  \STATE \textbf{Input:} $A$
  \STATE $a \gets [35.123, 0.238]$
  \STATE $b \gets [0.607, 1.109, 1.009, 1.000]$
  \FOR{$i=1$ to 2}
    \STATE $X_i = 0.5(a_{i-1}X_{i-1} + (a_{i-1}X_{i-1})^{-1})$
  \ENDFOR
  \FOR{$i=3$ to 6}
    \STATE $X_i = 1.5(b_{i-3}X_{i-1}) - 0.5(b_{i-3}X_{i-1})^3$
  \ENDFOR
  \RETURN $X_6$
\end{algorithmic}
\end{algorithm}

\vspace{1em}  

\begin{algorithm}[H]
\caption{SIGN on GPU ($d=5000$, \texttt{DOUBLE})}
\label{SIGN_DOUBLE}
\begin{algorithmic}
  \STATE \textbf{Input:} $A$
  \STATE $a \gets [32.158, 0.355, 0.582, 0.932, 0.998, 1]$
  \FOR{$i=1$ to 6}
    \STATE $X_i = 0.5(a_{i-1}X_{i-1} + (a_{i-1}X_{i-1})^{-1})$
  \ENDFOR
  \RETURN $X_6$
\end{algorithmic}
\end{algorithm}

When the precision is double, the relative runtime ratio between inverse and matrix multiplication is not as high as that of single precision. Hence, the model prefers Newton's method more for double precision case, and \cref{SIGN_DOUBLE} only consists of Newton's iteration whereas \cref{SIGN_FLOAT} contains initial Newton's iterations and latter NewtonSchulz iterations.


\subsection{Generalization}
\label{Generalization}
In this section we verify the generalization guarantee that we had in \cref{prop:guarantee}. We test the algorithm learned in \cref{Foundalg-sqrt} to a different matrix distribution with the same limiting spectrum, where each entries of $X$ are sampled i.i.d. from $\text{Unif}[-\sqrt{3}, \sqrt{3}]$ instead of $\mathcal{N}(0, 1)$. We denote the distribution as WishartUnif. In this case, the two spectrum converges to the same spectrum in distribution due to Marchenko-Pastur (\cref{LimitingDistribution-same}). 
\begin{figure}[ht] 
\centering
    \begin{subfigure}[t]{0.47\textwidth}
        \includegraphics[width=\linewidth]{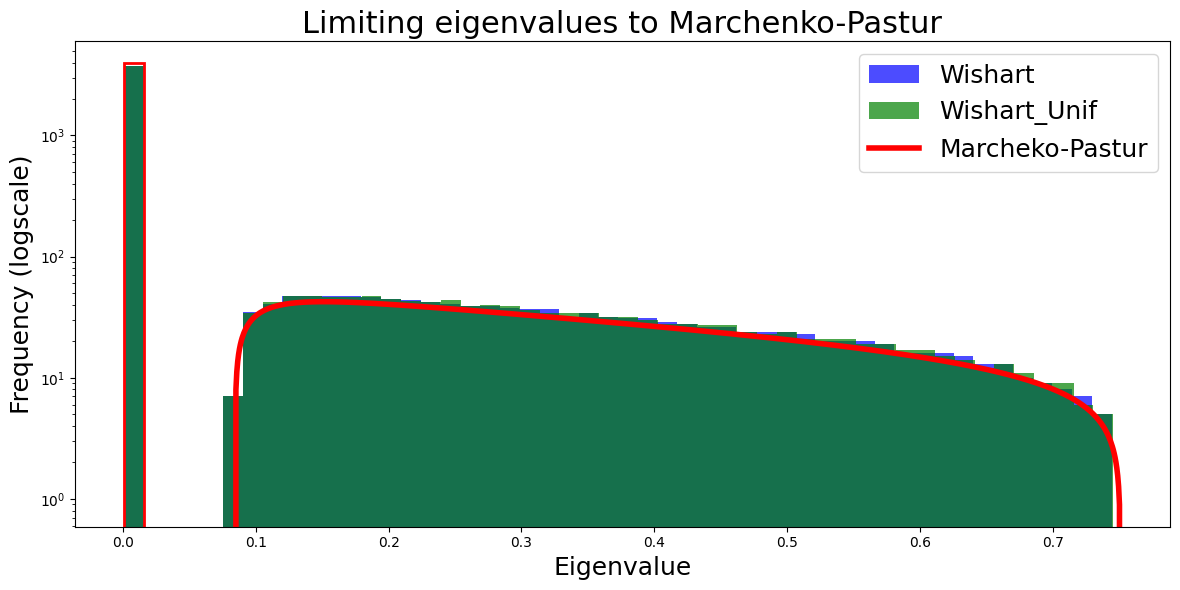}
        \caption{Same limiting eigenvalue distribution}
        \label{LimitingDistribution-same}
    \end{subfigure}
    \hfill
    \begin{subfigure}[t]{0.47\textwidth}
        \includegraphics[width=\linewidth]{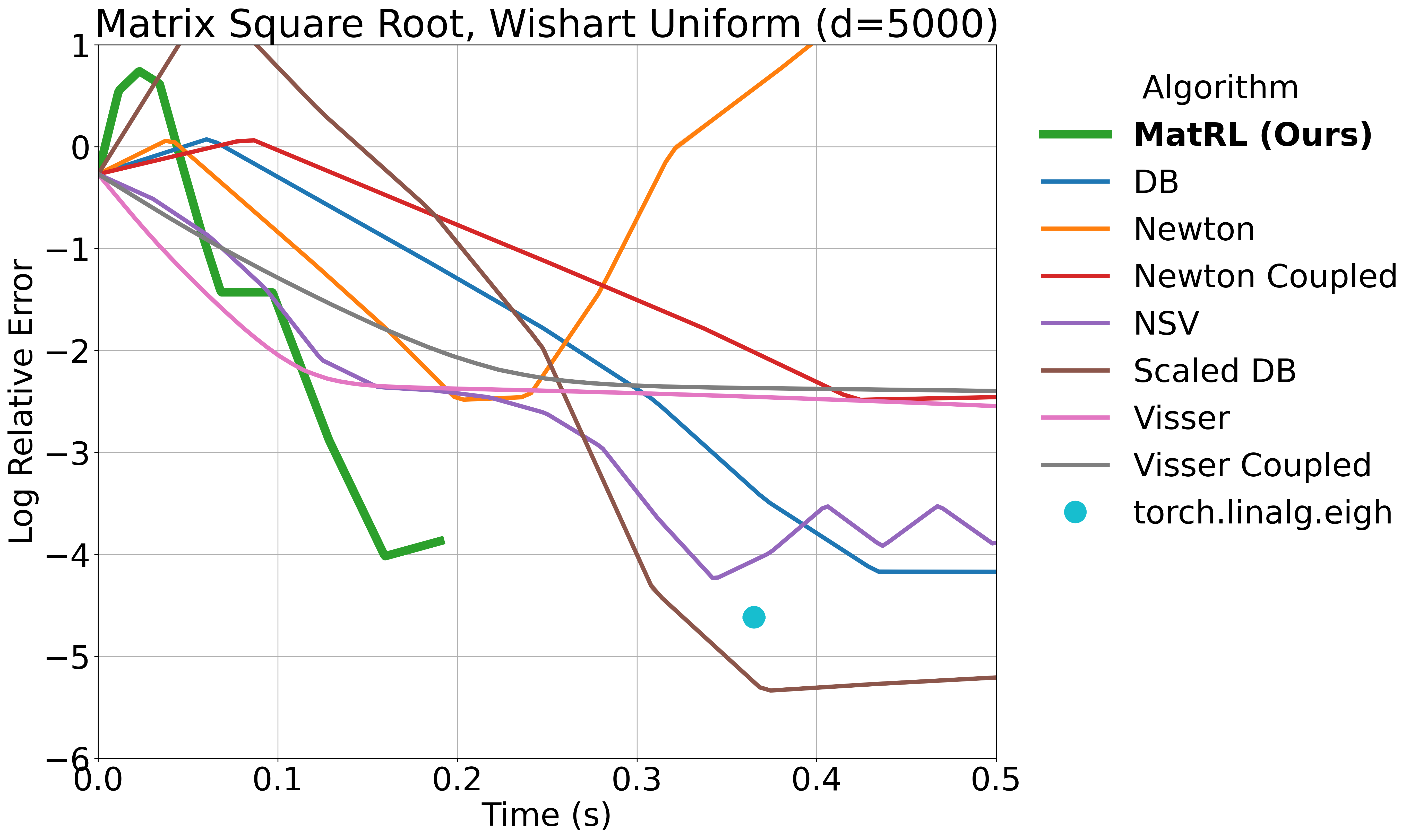}
        \caption{Testing \cref{Foundalg-sqrt}
        on a different distribution}
        \label{Generalizing}
    \end{subfigure}
    \caption{How generalization guarantee in \cref{prop:guarantee} works. Here we have two different random matrix distributions with the same limiting spectrum. Even though they are sampled from different distributions, the limiting spectrum coincide and the learned algorithm generalizes.}
    \label{fig:generalize}
\end{figure}

Due to \cref{prop:guarantee}, we expect the learned algorithm in \cref{Foundalg-sqrt} to work as well for matrices sampled from WishartUnif. \cref{Generalizing} indeed shows that it is true: when we compare the plot in \cref{fig:diffmatfunc} and \cref{Generalizing}, the two curves look almost identical - meaning the loss curve for two tests coincide, and generalization indeed happened.
\subsection{A Real World Example: CIFAR-10 and ZCA Whitening}
A common application of computing matrix square roots, and inverse-square, roots is ZCA whitening of images \cite{bell1997independent}. The ZCA whitening, also known as Mahalanobis whitening, decorrelates (or whitens) data samples using the inverse-square root of an empirical covariance matrix. We apply MatRL to learn an algorithm to compute square roots and inverse-square roots simultaneously using coupled iterations on empirical covariances of CIFAR-10 \cite{krizhevsky2009learning} images. CIFAR-10 dataset is released with an MIT license. Specifically, the random input matrix is $\hat\Sigma = \frac{1}{n} X^T X$  where $X \in \mathbb{R}^{n \times d}$ is a random batch of $n$ CIFAR-10 images. The algorithm discovered by MatRL achieves relative error close to machine precision significantly faster ($\sim$ 1.8x) than baselines in terms of wall-clock time, see \Cref{tab:time_to_eps}.
\newcolumntype{Y}{>{\centering\arraybackslash}X}
\begin{table}[ht]
\centering
\caption{Time to reach machine precision (in seconds) on CIFAR-10 matrix square root.}
\label{tab:time_to_eps}
\begin{tabularx}{\textwidth}{*{6}{Y}}
\toprule
\textbf{MatRL (Ours)} & DB    & Newton & Newton Coupled & NSV   & Scaled DB \\
\midrule
\textbf{1.04 \(\pm\) 0.0318} & 2.30 \(\pm\) 0.0561 & 1.89 \(\pm\) 0.0011 & 3.81 \(\pm\) 0.0004 & 3.58 \(\pm\) 0.0308 & 2.30 \(\pm\) 0.0630 \\
\bottomrule
\end{tabularx}
\end{table}
\vspace{-0.5cm}

\section{Conclusion}
\label{sec:conclusion}
In this paper we propose MatRL, an MCTS-based automated solution to find iterative algorithms for matrix function computation. We showed that we can generate an algorithm specifically tailored for a specific input matrix distribution and compute environment, and the found algorithm is guaranteed to generalize to different matrix distributions with the same limiting spectrum under certain assumptions. We verify our findings with experiments, showing MatRL found algorithms that are faster than existing baselines and competitive to standard torch library.

Our current work has a few limitations. First, the input matrix is restricted to symmetric random matrices, and we only have a guarantee that it will work for random matrix distributions with the same limiting spectrum. Another limitation is that we cannot fully grasp the numerical instability that may occur in matrix operations, because we use spectrum as states for the environment. We believe this is the main reason why some MCTS runs diverge at test stage. At last, our current implementation is on symmetric matrices. We believe it would be clear to extend the setup to nonsymmetric matrices using singuar values. Overcoming these weaknesses is a promising way to improve MatRL. 

Our work has many future directions. One extension of this work could be: can we find novel iterative algorithms to compute $f(A)b$ for a matrix function $f$ and square matrix $A \in \mathbb{R}^{n\times n }$, vector $b \in \mathbb{R}^{n}$. Another interesting direction is not fixing the matrix iterations beforehand, but making the RL agent discover novel iterations that can ensure stability, such as the Denman-Beavers iteration. At last, automatic discovery of optimal parameters of optimization algorithms such as Muon \cite{jordan2024muon} with a different reward (e.g. validation error) could be an interesting direction.

\newpage
\bibliographystyle{plain} 
\bibliography{Styles/neurips_2025.bib}

\newpage
\appendix
\renewcommand{\theproposition}{A.\arabic{proposition}}
\renewcommand{\thelemma}{A.\arabic{lemma}}

\section{Technical Proofs}
\label{sec:proofs}

We first show that the environment can be parametrized by the spectrum.

\begin{proposition}
Let $\{f_k(X, Y, A)\}_{k=1}^{m}$ is a set of congruence invariant diagonal preserving matrix functions that take $X, Y, A$ as input and outputs $(X', Y')$, $\mathcal{L}(X,A)$ be a congruence invariant matrix function that takes $X, A$ as input and outputs a scalar, and $X_0, Y_0, \{X_k\}_{k=1}^{m+1}$ satisfy
\[
(X_0, Y_0) = (A, I) \quad \text{or} \quad (I, A)
\]
and
\[
(X_{k+1}, Y_{k+1}) = f(X_{k},Y_k, A), \quad k \in [m].
\]
At last, write $A = UD_AU^{T}$. Then, the following properties hold:\\
i) $X_{k}, Y_k \sim A$ for all $k \in [m+1]$.\\
ii) Write $X_k = UP_kU^{T}$, $Y_k  = UQ_kU^{T}$. Then, $P_{k+1}$, $Q_{k+1}$ depends only on $P_k$, $Q_k$ and $D_A$.\\
iii) The loss $\mathcal{L}(X_k, A)$ only depends on $P_k, D_A$.
\end{proposition}
\begin{proof}
i) We prove by induction. \\
If $k = 0$, we know that $X_0, Y_0 = A, I$ or $I, A$ hence they are similar with $A$.\\
Say $X_t, Y_t$ are similar with $A$. Then $X_t = UD_1U^{T}, Y_t = UD_2U^{T}$ for diagonal $D_1, D_2$. Now we can see that 
\[
f(X_t, Y_t, A) = f(UD_1U^{T}, UD_2U^{T}, UD_AU^{T}) = Uf(D_1, D_2, D_A)U^{T}  = (UD_1'U^{T}, UD_2'U^{T}),
\]
for some diagonal $D_1', D_2'$. The second equality comes from congruent invariance, and the third equality follows from diagonal preservence. Hence, $X_{t+1} = UD_1'U^{T}$ and $X_{t+1}$ and $A$ are similar, $Y_{t+1} = UD_2'U^{T}$ and $Y_{t+1}$ and $A$ are also similar.\\
ii) From the proof of i) we know that $f(D_1, D_2, D_A) = (D_1', D_2')$, where $D_1, D_2, D_A, D_1', D_2'$ are the spectrum of $X_{t}, Y_{t}, A, X_{t+1}, Y_{t+1}$, respectively.\\
iii) We know $X_k$ and $A$ are similar. Let $X_k = UP_kU^{T}$ and $A = UD_AU^{T}$. $\mathcal{L}(UP_kU^{T}, UD_AU^{T}) = \mathcal{L}(P_k, D_A)$ by congruent invariance.
\end{proof}

Next we show that all iterations that we deal with in the paper is congruent invariant diagonal preserving. See \cref{table:baselines} and \cref{table:actionlist} for the iterations of interest. 
\begin{lemma}
Assume $f(X_1, X_2, \cdots X_k)$, $g(X_1, X_2, \cdots X_k)$ are congruent invariant diagonal preserving. Then, $f+g$, $fg$, $f^{-1}$ are all congruent invariant diagonal preserving.
\end{lemma}
\begin{proof}
Diagonal preserving is simple, as sum, multiple, inverse of diagonal matrices are diagonal. Let's show congruence invariance. 
\[
f(QX_1Q^{T}, \cdots QX_kQ^{T}) + g(QX_1Q^{T}, \cdots QX_kQ^{T}) = Qf(X_1, \cdots X_k)Q^{T} + Qg(X_1,\cdots X_k)Q^{T}
\]
and
\[
Qf(X_1, \cdots X_k)Q^{T} + Qg(X_1,\cdots X_k)Q^{T} = Q(f(X_1, \cdots X_k)+g(X_1,\cdots X_k))Q^{T}.
\]
\[
f(QX_1Q^{T}, \cdots QX_kQ^{T})g(QX_1Q^{T}, \cdots QX_kQ^{T}) = Qf(X_1, \cdots X_k)Q^{T}Qg(X_1,\cdots X_k)Q^{T}
\]
and
\[
Qf(X_1, \cdots X_k)Q^{T}Qg(X_1,\cdots X_k)Q^{T} = Q(f(X_1, \cdots X_k)g(X_1,\cdots X_k))Q^{T}.
\]
\[
f(QX_1Q^{T}, \cdots ,QX_kQ^{T})^{-1} = (Qf(X_1, \cdots X_k)Q^{T})^{-1} = Qf(X_1, \cdots X_k)^{-1}Q^{T}.
\]
\end{proof}
\begin{proposition}
\label{prop:cgdi_interest}
Say $f(X_1, X_2, \cdots X_k)$ is a rational function of $X_1, \cdots X_k$, i.e. $f(X_1, X_2, \cdots X_k) = P(X_1, \cdots X_k)Q(X_1, \cdots X_k)^{-1}$ for polynomials $P, Q$. Then $f$ is congruent invariant diagonal preserving.
\end{proposition}
\begin{proof}
We know that $P, Q,$ and hence $PQ^{-1}$ is congruent invariant diagonal preserving from lemma 1.
\end{proof}
As all iterations in \cref{table:actionlist} or \cref{table:baselines} are rational functions, the iterations of interest are congruent invariant diagonal preserving.

At last we show \cref{prop:guarantee}: our found algorithm will generalize to the matrices drawn from the distribution with the same limiting eigenvalue spectrum.

\begin{proposition}
(Generalization of the discovered algorithm, \cref{prop:guarantee} of the main paper) Say we have a sequence of symmetric random matrix distributions $\mathcal{P}_m, \mathcal{Q}_m$ defined in $\mathbb{R}^{m \times m}$, and denote random matrices sampled from $\mathcal{P}_m$, $\mathcal{Q}_m$ as $X, Y$. Let the empirical eigenvalue value distribution of $X \sim \mathcal{P}_m, Y \sim \mathcal{Q}_m$ be $\mu_m(X), 
 \nu_m(Y)$, and their support be $S_m, S'_m$, respectively. Now, suppose\\
(i) (Identical limiting distribution) 
\[
\mathbb{P}(\mu_m(X) \Rightarrow \mu^{*}) = \mathbb{P}(\nu_m(Y) \Rightarrow \mu^{*}) = 1,
\]
i.e. both $\mu_m(X)$ and $\nu_m(Y)$ converges weakly to a common distribution $\mu^{*}$ with probability 1.\\
(ii) (Interval support of the limiting distribution) The support of $\mu^{*}$ is an interval $[a, b]$. \\
(iii) (Convergence of support) We have
\[
\lim_{m \rightarrow \infty} \mathbb{P}(S_m \subseteq [a-\epsilon, b+\epsilon]) = \lim_{m \rightarrow \infty} \mathbb{P}(S'_m \subseteq [a-\epsilon, b+\epsilon]) = 1,
\]
for all $\epsilon > 0$.\\
With the assumptions above, let $f$ be the matrix function we would like to compute, $f^{*}_k$ the step $k$ transformation of eigenvalues of the algorithm found by \cref{MatRL}, and $\mathcal{L}$ be the loss. Assume $f, f_k^{*}, L$ are continuous in $[a - \epsilon_0,b + \epsilon_0]$ for some $\epsilon_0 > 0$. Write the empirical loss of the random matrix $X$ as
\[
\mathcal{L}_k(X) = \frac{1}{m} \sum_{i=1}^{m} L(f(\lambda_i), f_k^{*}(\lambda_i)),
\]
where $\lambda_i$ are eigenvalues of $X$.\\
Then, there exists $M_{\epsilon, \delta}$ such that
\[
m \geq M_{\epsilon, \delta} \Rightarrow \mathbb{P}_{X \sim \mathcal{P}_m, Y \sim \mathcal{Q}_m} [|\mathcal{L}_k(X) - \mathcal{L}_k(Y)| < \epsilon] \geq 1 - \delta.
\]
\end{proposition}
\begin{proof}
We write 
\[
\mathcal{L}^{*} = \int L(f(\sigma), f_k(\sigma)) d\mu^{*}(\sigma).
\]
We would like to show that for sufficiently large $m$, $|\mathcal{L}_{k}(X) - \mathcal{L}^{*}| < \epsilon/2$ with high probability. As $\mathbb{P}(\mu_m(X) \Rightarrow \mu^{*}) = 1$, we know that with probability 1,
\[
\lim_{m \rightarrow \infty} \int f d\mu_m(X) = \int fd\mu^{*}
\]
for all continuous bounded $f$. We shall extend $L(f(x), f_k(x))$ in a way that it is continuous bounded in $\mathbb{R}$ and the ``difference" is small.\\
First, we know that $L(f(x), f_k(x))$ is continuous in $[a-\epsilon_0, b+\epsilon_0]$. Say
\[
A = \max_{x \in [a - \epsilon_0 , b + \epsilon_0], y \in [a, b]} |L((f(x), f_k(x))| + |L((f(y), f_k(y))|. 
\]
Now, choose $\epsilon' = \max\{\epsilon_0, \epsilon/8A\}$ (when $A = 0$ we just have $\epsilon' = \epsilon_0$). With the chosen $\epsilon'$, define $\tilde{L}$ as
\[
\tilde{L}(x) = \begin{cases}
    L(f(x), f_k(x)) \quad if \quad x \in [a, b]\\
    (x - a + \epsilon') \frac{L(f(a), f_k(a))}{\epsilon'} \quad if \quad x \in [a - \epsilon', a]\\
    (-x + b + \epsilon') \frac{L(f(b), f_k(b))}{\epsilon'} \quad if \quad x \in [b, b + \epsilon']\\
    0 \quad if \quad x \in (-\infty, a-\epsilon'], [b+\epsilon', \infty),
\end{cases}
\]
which is a bounded continuous function in $\mathbb{R}$. 
At last, choose $M_1$ sufficiently large so that $m \geq M_1$ implies
\[
|\int \tilde{L} d\mu_m(X) - \int \tilde{L}d\mu^{*}| < \epsilon/4
\]
with probability 1 and
\[
\mathbb{P}(S_m \subseteq [a - \epsilon', b + \epsilon']) \geq 1 - \delta/2.
\]
Such $m$ exists because of assumptions (i) and (iii). Now, we know that
\[
\mathcal{L}_k(X) = \int L(f(x),f_k(x)) \mu_m(X).
\]
Moreover, the function $|\tilde{L}(x) - L(f(x),f_k(x))| \leq A$ for $x \in S_m$. This is because \[|\tilde{L}(x)|\leq \max_{x \in [a, b]} |L(f(x), f_k(x))|, \quad S_m \subseteq [a - \epsilon' , b + \epsilon'] \subseteq [a - \epsilon_0, b + \epsilon_0].\]
As $\tilde{L}(x) - L(f(x),f_k(x)) = 0$ for $x \in [a,b]$, the value
\[
|\int \tilde{L}d\mu_m(X) - \int L(f(x),f_k(x))d\mu_m(X)| \leq 2\epsilon'A \leq \epsilon/4
\]
with probability at least $1 - \delta/2$. Hence, when $m \geq M_1$, 
\[
|\mathcal{L}_k(X) - \int \tilde{L}d\mu^{*}| = |\mathcal{L}_k(X) - \mathcal{L}^{*}| < \epsilon/2
\]
with probability at least $1 - \delta/2$. We can do the same argument for $Y$ to find $M_2$. Take $M_{\epsilon, \delta} = \max\{M_1, M_2\}$. Using union bound, we can see the probability that both $|\mathcal{L}_k(X) - \mathcal{L}^{*}| < \epsilon/2$ and $|\mathcal{L}_k(Y) - \mathcal{L}^{*}| < \epsilon/2$ happens is at least $1 - \delta$. Hence, $\mathbb{P}_{X \sim \mathcal{P}_m, Y \sim \mathcal{Q}_m} [|\mathcal{L}_k(X) - \mathcal{L}_k(Y)| < \epsilon] \geq 1 - \delta$.
\end{proof}
\newpage

\section{List of Used Matrix Iterations}
\label{sec:matrix functions and iterations}

We first present a table that shows different types of baseline algorithms used in the paper with references. This table is a superset of \cref{table:listbaselines}. 
\begin{table}[H]
\centering
\caption{List of baselines}
\label{table:listbaselines_full}
\begin{tabular}{
  ll
  S[table-format=3.1(2)]
  S[table-format=3.1(2)]
  S[table-format=1.2]
}
\toprule
Matrix function & List of baselines \\
\midrule
Inverse & NS \cite{pan1991improved}, Chebyshev \cite{li2011chebyshev}\\
Sign & Newton \cite{higham2008functions}, NS \cite{schulz1933iterative}, ScaledNewton \cite{byers2008new}, ScaledNS \cite{chen2014newton}, Halley\cite{nakatsukasa2010optimizing}\\
Square root& DB\cite{denman1976matrix}, NSV\cite{higham1997stable}(2.6), Scaled DB\cite{higham1997stable}, Visser\cite{higham2008functions}, Newton\cite{hoskins1979faster}\\
1/3 - root & Iannazzo \cite{iannazzo2006newton}, Visser \cite{higham2008functions}, Newton \cite{iannazzo2006newton}(1.2)\\
\bottomrule
\end{tabular}
\end{table}
\subsection{Iterative methods associated with inverse}
We have two different baselines for inverse. One is Newton's method proposed by Schulz, which is the iteration
\[
\textbf{\textrm{(InvNewton)}} \quad \quad X_{k+1} = 2X_k - X_kAX_k.
\]
For an appropriate initialization, the norm $\lVert I - AX_k \rVert_2$ will converge quadratically to zero. This is because we can write 
\[
I - AX_{k+1} = I - 2AX_k + AX_kAX_k = (I - AX_k)^2.
\]
Another baseline is applying Chebyshev's iteration to the function $X^{-1} - A$. We have
\[
\textbf{\textrm{(InvChebyshev)}} \quad \quad X_{k+1} = 3X_k - 3X_kAX_k + X_kAX_kAX_k.
\]
With similar logic we can obtain $I - AX_{k+1} = (I - AX_k)^3$. Hence at each iteration the error decreases cubically. The drawback is that Chebyshev's method needs at least three matrix-matrix multipications each iteration. 
\subsection{Iterative methods associated with sign}
The simplest method to compute matrix sign is Newton's method, where the iteration is given as
\[
\textbf{\textrm{(SignNewton)}} \quad\quad X_{k+1} = \frac{1}{2}(X_k + X_k^{-1}).
\]
The NewtonSchulz variant avoids computing inverse by using the iteration
\[
X_{k+1} = \frac{1}{2}(3X_k - X_kX_k^{T}X_k),
\]
hence for symmetric matrices
\[
\textbf{\textrm{(SignNewtonSchulz)}} \quad\quad X_{k+1} = \frac{1}{2}(3X_k - X_k^3).
\]
Newton's method has scaled variants, where we do
\[
\textbf{\textrm{(SignScaledNewton)}} \quad\quad X_{k+1} = \frac{1}{2}(\mu_kX_k + (\mu_kX_k)^{-1}),
\]
for specific $\mu_k$. Our baseline is the one proposed in \cite{byers2008new}. Here $\mu_k$ is defined as the following: we let $a, b$ be constants that satisfy $a \leq \sigma_n \leq \sigma_1 \leq b$ for the singular values of $A$. Then 
\[
\mu_0 = \frac{1}{\sqrt{ab}},\quad \mu_1 = \sqrt{\frac{2}{\sqrt{a/b}+\sqrt{b/a}}},\quad \mu_k = \sqrt{\frac{2}{\mu_{k-1}+\mu_{k-1}^{-1}}}, k \geq 2.
\]
$a$ and $b$ can be obtained by computing $\lVert A \rvert_2$ and $\lVert A^{-1} \rvert_2^{-1}$.
NewtonSchulz method may also have variants: a recent variant in \cite{chen2014newton} scales each $X_k$ as
\[
\textbf{\textrm{(SignScaledNewtonSchulz)}} \quad\quad X_{k+1} = \frac{3}{2}\rho_kX_k - \frac{1}{2}(\rho_kX_k)^3,
\]
where $X_0$ = $A / \lambda_{|max|}(A)$, $x_0 = \lambda_{|min|}(A)/\lambda_{|max|}(A)$ and
\[
\rho_k = \sqrt{\frac{3}{1+x_0+x_0^2}}, \quad x_{k+1} = \frac{1}{2}\rho_kx_k(3 - \rho_k^2x_k^2).
\]
Halley's method uses a rational approximation of sign function to compute the matrix sign. The iteration is written as
\[
\textbf{\textrm{(SignHalley)}} \quad\quad X_{k+1} = X_k(a_kI+b_kX_k^2)(I+c_kX_k^2)^{-1},
\]
where default Halley's iteration uses $a = c = 3, b = 1$ and the scaled Halley in \cite{nakatsukasa2010optimizing} uses certain optimal coefficients.

Newton variant is essentially a variant of Newton's method where we do
\[
\textbf{\textrm{(SignNewtonVariant)}} \quad\quad X_{k+1}  = 2X_k(I+X_k^2)^{-1}.
\]
This is the inverse of SignNewton, and $X_k$ converges to $\text{sign}(A)^{-1}$, which is $\text{sign}(A)$ when $A$ is invertible.

\subsection{Iterative methods associated with square root}

The simplest method in this case is also the Newton's method, 
\[
\textbf{\textrm{(SqrtNewton)}} \quad\quad X_{k+1} = \frac{1}{2}(X_k + X_k^{-1}A).
\]
The above method can be unstable, which led to the development of coupled iterations. Denman-Beavers iteration uses the following coupled iterationn of $X_k$ and $Y_k$: Denman-Beavers is initialized with $X_0 = A, Y_0 = I$ and iteratively applies
\[
\textbf{\textrm{(SqrtDenmanBeavers)}} \quad\quad \begin{cases}
X_{k+1} = \frac{1}{2}(X_k + Y_k^{-1})\\
Y_{k+1} = \frac{1}{2}(Y_k + X_k^{-1}).
\end{cases}
\]
Here, $X_k \rightarrow A^{1/2}$ and $Y_k \rightarrow A^{-1/2}$. There is a variant of Denman-Beavers that avoids computing matrix inverse - introduced in \cite{higham1997stable}, the iteration writes
\[
\textbf{\textrm{(SqrtNewtonSchulzVariant)}} \quad\quad \begin{cases}
X_{k+1} = \frac{1}{2}(3X_k - X_kY_kX_k)\\
Y_{k+1} = \frac{1}{2}(3Y_k - Y_kX_kY_k).
\end{cases}
\]
Like ScaledNewton, we have a scaled variant of Denman-Beavers. The scaling we use is a variant of Byer's scaling \cite{byers1987solving} introduced in \cite{higham1997stable}. The iteration is given as
\[
\textbf{\textrm{(SqrtScaledDenmanBeavers)}} \quad\quad \begin{cases}
\gamma_k = |\det X_k \det Y_k|^{-1/2n}\\
X_{k+1} = \frac{1}{2}(\gamma_kX_k + (\gamma_kY_k)^{-1})\\
Y_{k+1} = \frac{1}{2}(\gamma_kY_k + (\gamma_kX_k)^{-1}).
\end{cases}
\]
The cost of computing $\gamma_k$ is negligible when we use decomposition methods such as LU decomposition or Cholesky to compute matrix inverse. 

At last, there is the fixed-point iteration, which we will denote as the Visser iteration \cite{higham2008functions}. The Visser iteration is  given as
\[
\textbf{\textrm{(SqrtVisser)}} \quad\quad X_{k+1} = X_k + \frac{1}{2}(A - X_k^2).
\]

\subsection{Iterative methods associated with 1/3-root}

There are a number of stable methods to compute matrix $p$-th root (see \cite{iannazzo2006newton} for different methods). We use the following method as a baseline: initialize $X_0 = I, Y_0 = A$ and  
\[
\textbf{\textrm{(prootIannazzo)}} \quad\quad
\begin{cases}
X_{k+1} = X_k(\frac{2I + Y_k}{3})\\
Y_{k+1} = (\frac{2I + Y_k}{3})^{-3}Y_k
\end{cases}
\]
With this iteration, $X_k \rightarrow A^{1/3}$ and $Y_k \rightarrow I$. We have Newton's method and Visser's iteration as we had for square root:
\[
\textbf{\textrm{(prootNewton)}} \quad\quad X_{k+1} = (2X_k+X_kA^{-2})/3
\]
is the Newton's method, and 
\[
\textbf{\textrm{(prootVisser)}} \quad\quad X_{k+1} = X_k + \frac{1}{3}(A - X_k^3)
\]
becomes Visser's iteration.
\subsection{Summary}
We present a table of the matrix functions and iterations that we used.
\begin{table}[h!]
\label{table:baselines}
\caption{Iterative methods for computing matrix inverse, sign, square root, and 1/3-root.}
\centering
\renewcommand{\arraystretch}{1.4}
\begin{tabular}{@{}ll@{}}
\toprule
\textbf{Method} & \textbf{Iteration Formula} \\
\midrule
\multicolumn{2}{c}{\textbf{Methods for Inverse}} \\
\midrule
Newton (Schulz) & $X_{k+1} = 2X_k - X_k A X_k$ \\
Chebyshev & $X_{k+1} = 3X_k - 3X_k A X_k + X_k A X_k A X_k$ \\
\midrule
\multicolumn{2}{c}{\textbf{Methods for Sign}} \\
\midrule
Newton & $X_{k+1} = \frac{1}{2}(X_k + X_k^{-1})$ \\
NewtonSchulz & $X_{k+1} = \frac{1}{2}(3X_k - X_k^3)$ \\
ScaledNewton & $X_{k+1} = \frac{1}{2}(\mu_k X_k + (\mu_k X_k)^{-1})$ \\
ScaledNewtonSchulz & $X_{k+1} = \frac{3}{2} \rho_k X_k - \frac{1}{2} (\rho_k X_k)^3$ \\
Halley & $X_{k+1} = X_k (a_k I + b_k X_k^2)(I + c_k X_k^2)^{-1}$ \\
NewtonVariant & $X_{k+1} = 2X_k(I + X_k^2)^{-1}$ \\
\midrule
\multicolumn{2}{c}{\textbf{Methods for SquareRoot}} \\
\midrule
Newton & $X_{k+1} = \frac{1}{2}(X_k + X_k^{-1} A)$ \\
DenmanBeavers &
$\begin{cases}
X_{k+1} = \frac{1}{2}(X_k + Y_k^{-1}) \\
Y_{k+1} = \frac{1}{2}(Y_k + X_k^{-1})
\end{cases}$ \\
NewtonSchulzVariant &
$\begin{cases}
X_{k+1} = \frac{1}{2}(3X_k - X_k Y_k X_k) \\
Y_{k+1} = \frac{1}{2}(3Y_k - Y_k X_k Y_k)
\end{cases}$ \\
ScaledDenmanBeavers &
$\begin{cases}
\gamma_k = |\det X_k \det Y_k|^{-1/2n} \\
X_{k+1} = \frac{1}{2}(\gamma_k X_k + (\gamma_k Y_k)^{-1}) \\
Y_{k+1} = \frac{1}{2}(\gamma_k Y_k + (\gamma_k X_k)^{-1})
\end{cases}$ \\
Visser & $X_{k+1} = X_k + \frac{1}{2}(A - X_k^2)$ \\
\midrule
\multicolumn{2}{c}{\textbf{Methods for 1/3-th Root}} \\
\midrule
Iannazzo &
$\begin{cases}
X_{k+1} = X_k \left(\frac{2I + Y_k}{3}\right) \\
Y_{k+1} = \left(\frac{2I + Y_k}{3}\right)^{-3} Y_k
\end{cases}$ \\
Newton & $X_{k+1} = \frac{1}{3}(2X_k + X_k A^{-2})$ \\
Visser & $X_{k+1} = X_k + \frac{1}{3}(A - X_k^3)$ \\
\bottomrule
\end{tabular}
\end{table}

\newpage
\section{Experimental Details}
\label{sec:expdetails}

\subsection{Detailed explanation of \cref{MatRL}}
\cref{MatRL} is summarized in \cref{sec:SearchingStrategy}. Here we explain how each subroutine $Expandable$, $Best_{UCB}$, $ExpandNode$, $SampleRolloutList$, and $backpropagate$ is implemented.

To begin with, we have two important flags at each state. One flag is IsTransitionable: if the iteration type and all parameters for that iteration is fixed, we set IsTransitionable(s) = True. Else, IsTransitionable(s) = False. Another flag is IsCoupled: for the root state, IsCoupled = True. If you use a coupled iteration at a state where IsCoupled = True, IsCoupled = True at the next state also. If you use an iteration that is not coupled, we set IsCoupled = False. When IsCoupled = True, you can do either coupled or uncoupled iteration. When IsCoupled = False, you can either do the "coupling" iteration (that would be specified later for each matrix function) or an iteration that is not coupled. Whenn you do the coupling iteration, IsCoupled = True for the next state, else IsCoupled = False. A table that summmarizes the transition of IsCoupled is as below.

\begin{table}[htbp]
\centering
\caption{State transition of IsCoupled}
\begin{tabular}{ccc}
\toprule
\textbf{Current IsCoupled} & \textbf{Iteration Type} & \textbf{Next IsCoupled} \\
\midrule
True  & Coupled iteration         & True \\
True  & Uncoupled iteration       & False \\
False & Coupling iteration        & True \\
False & Non-coupling iteration    & False \\
\bottomrule
\end{tabular}
\end{table}

$Expandable(s)$ is a method that determines whether it is possible to expand a child node from current node $s$. If $IsTransitionable(s) = True$, the possible choice of next action becomes a discrete set of iterations. Hence $Expandable(s) = True$ if $IsCoupled(s) = False$ and number of children of $s$ $<$ number of iterations that are not coupled, or $IsCoupled(s) = True$ andd number of children of $s < $ number of iterations - 1. We subtract 1 because we will not expand with coupling iteration. If $IsTransitionable = False$, we expand with a continuous variable hence we do progressive widening. If number of children of $s < C_{pw}N(s)^{\alpha_{pw}}$ where $n(s)$ is the number of visits for node $s$, we return True and else return False.  
\begin{algorithm}
\caption{$Expandable(s)$}
\begin{algorithmic}[1]
\IF{$IsTransitionable(s)$}
    \IF{$IsCoupled(s) = \text{False}$}
        \IF{number of children of $s$ $<$ number of \emph{non-coupled} iterations}
            \RETURN True
        \ELSE
            \RETURN False
        \ENDIF
    \ELSE
        \IF{$\text{num\_children}(s) < \text{num\_iterations} - 1$}
        \RETURN True
        \ELSE
        \RETURN False
        \ENDIF
    \ENDIF
\ELSE 
    \IF{number of children of $s$ $< C_{pw} n(s)^{\alpha_{pw}}$}
        \RETURN True
    \ELSE
        \RETURN False
    \ENDIF
\ENDIF
\end{algorithmic}
\end{algorithm}

$Best_{UCB}(s)$ is simple: Choose the child node $c$ with the maximal value of $V(c) + C_{ucb}\sqrt{\frac{\log n(s)+1}{n(c)}}$.

\begin{algorithm}
\caption{$Best_{UCB}(s)$}
\begin{algorithmic}[1]
\STATE \textbf{Input:} Node $s$ with children set $\mathcal{C}(s)$
\STATE \textbf{Parameters:} Exploration constant $C_{ucb}$
\STATE $best\_value \gets -\infty$
\STATE $best\_child \gets$ null
\FORALL{$c \in \mathcal{C}(s)$}
    \STATE $score \gets V(c) + C_{ucb} \sqrt{\frac{\log (n(s)) + 1}{n(c)}}$
    \IF{$score > best\_value$}
        \STATE $best\_value \gets score$
        \STATE $best\_child \gets c$
    \ENDIF
\ENDFOR
\RETURN $best\_child$
\end{algorithmic}
\end{algorithm}

$ExpandNode(s)$ depends on $IsTransitionable$. If $IsTransitionable(s) = True$, simply adding a node that hasn't  been visited is enough, because the children are discrete. If $IsTransitionable(s) = False$, the children can take continuous parameters. If $num\_child(s) \leq E$ for hyperparameter $E$, do random sampling in range $[lo, hi]$ that is prespecified. Else, find the child with the best value and sample near that parameter $p$. Specifically, with probability 0.05, sample uniformly at random from $[lo, hi]$. Else sample random uniform at a new interval $[lo,hi] \cap [p - stddev\_scale * (hi - lo) / 2.0, p + stddev\_scale * (hi - lo) / 2.0]$. $stddev\_scale = 1/\log(2+n(s))$ decays logarithmically with $n(s)$, the visit count. 

\begin{algorithm}
\caption{$ExpandNode(s)$}
\begin{algorithmic}[1]
\IF{$IsTransitionable(s)$}
    \STATE Add a new discrete child node to $s$
\ELSE
    \IF{$num\_child(s) \leq E$} 
        \STATE Sample $x \sim \mathcal{U}[lo, hi]$
        \STATE Add child node with parameter $x$
    \ELSE 
        \STATE $p \gets$ parameter of best-value child of $s$
        \STATE $w \gets (hi - lo)/2$
        \STATE $stddev\_scale \gets 1/\log(2 + n(s))$
        \STATE $r \sim \mathcal{U}[0,1]$
        \IF{$r < 0.05$}
            \STATE Sample $x \sim \mathcal{U}[lo, hi]$
        \ELSE
            \STATE Define interval $I = [lo, hi] \cap [p - stddev\_scale \cdot w,\ p + stddev\_scale \cdot w]$
            \STATE Sample $x \sim \mathcal{U}[I]$
        \ENDIF
        \STATE Add child node with parameter $x$
    \ENDIF
\ENDIF
\end{algorithmic}
\end{algorithm}

$SampleRolloutList()$ samples a baseline rollout algorithm that is consisted of mutiple iterations of well-working baselines such as scalednewton for sign or scaled Denman-Beavers for matrix square root. If the rollout is coupled iteration but the current state is not coupled, we append the coupling iteration at the front of rollout.

At last, $Backpropogate(s)$ uses Bellman equation to update $V(s)$ in the path from root to $s$ and if $V(s_0)$ is updated, we update bestpath and bestrollout accordingly. 

\subsection{Experimental Environment}
All GPU based experiments were done in NVIDIA RTX A-6000 and CPU based experiments were done in AMD EPYC 7713 64-Core Processor. We repeated the experiments five times and picked  the best algorithm, and if the method diverged for five times we ran additional experiments to find a good algorithm.

\subsection{Hyperparameters}
Here we detail the hyperparameters in \cref{MatRL}: this includes basic parameters such as $\alpha$ in progressive widening, list of possible actions for each matrix function, and $RolloutList$ for each matrix function. 

We set $C_{pw} = 2, \alpha_{pw} = 0.3,C_{ucb} = 5, E = 5, \epsilon_{tol} = 1e-6$ and $1e-11$ for the experiments. The loss function and RolloutList for each matrix function is as below:

\begin{table}[h]
\centering
\caption{Loss function, action list, and rollout list for each matrix function}
\begin{tabular}{l p{2.5cm} p{4cm} p{4cm}}
\toprule
\textbf{Function} & \textbf{Loss Function} & \textbf{ActionList} & \textbf{RolloutList} \\
\midrule

\textbf{Inv} 
& $\frac{\lVert AX - I \rVert_F}{\lVert A \rVert_F}$ 
& \texttt{[Inv\_NS, Inv\_Chebyshev]} 
& \texttt{[Inv\_NS, Inv\_Chebyshev]} \\

\addlinespace
\textbf{Sign} 
& $\frac{\lVert X^2 - I \rVert_F}{\lVert A \rVert_F}$ 
& \texttt{[Sign\_NS, Sign\_Newton, Sign\_Quintic, Sign\_Halley]} 
& \texttt{[Sign\_ScaledNS, Sign\_ScaledNewton, Sign\_Halley]} \\

\addlinespace
\textbf{Sqrt} 
& $\frac{\lVert X^2 - A \rVert_F}{\lVert A \rVert_F}$ 
& \texttt{[Sqrt\_DB, Sqrt\_NSV, Sqrt\_Visser, Sqrt\_VisserCoupled, Sqrt\_Coupling]} 
& \texttt{[Sqrt\_ScaledDB, Sqrt\_NSV]} \\

\addlinespace
\textbf{Proot} 
& $\frac{\lVert X^3 - A \rVert_F}{\lVert A \rVert_F}$ 
& \texttt{[Proot\_Newton, Proot\_Visser, Proot\_Iannazzo, Proot\_Coupling]} 
& \texttt{[Proot\_Newton, Proot\_Visser, Proot\_Iannazzo]} \\

\bottomrule
\end{tabular}
\end{table}

Each iteration in ActionList is parametrized to have tunable parameters. A full table denoting how each action is parameterized is as \cref{table:actionlist}. 

\begin{table}[h!]
\caption{How actions are parametrized}
\centering
\renewcommand{\arraystretch}{1.4}
\begin{tabular}{@{}lll@{}}
\toprule
\textbf{Method} & \textbf{Iteration Formula} & \textbf{Parameter Range} \\
\midrule
\multicolumn{3}{c}{\textbf{Actions for Inverse}} \\
\midrule
Newton (Schulz) & $X_{k+1} = a_kX_k - b_kX_k A X_k$ & $a_k, b_k \in [0, 5]$ \\
Chebyshev & $X_{k+1} = a_kX_k - b_kX_k A X_k + c_kX_k A X_k A X_k$ & $a_k, b_k, c_k \in [0, 5]$ \\
\midrule
\multicolumn{3}{c}{\textbf{Actions for Sign}} \\
\midrule
Newton & $X_{k+1} = \frac{1}{2}(a_kX_k + (a_kX_k)^{-1})$ & $a_k \in [0, 40]$ \\
NewtonSchulz & $X_{k+1} = X_k + a_k(b_kX_k - (b_kX_k)^3)$ & $a_k, b_k \in [0, 5]$ \\
Quintic & $X_{k+1} = a_kX_k+b_kX_k^3+c_kX_k^5$ & $a_k, b_k, c_k \in [0, 5]$ \\
Halley & $X_{k+1} = X_k (a_k I + b_k X_k^2)(I + c_k X_k^2)^{-1}$ & $a_k, b_k, c_k \in [0, 40]$ \\
\midrule
\multicolumn{3}{c}{\textbf{Actions for SquareRoot}} \\
\midrule
DenmanBeavers &
$\begin{cases}
X_{k+1} = \frac{1}{2}(a_kX_k + (b_kY_k)^{-1}) \\
Y_{k+1} = \frac{1}{2}(b_kY_k + (a_kX_k)^{-1})
\end{cases}$ & $a_k, b_k \in [0, 50]$ \\
NewtonSchulzVariant &
$\begin{cases}
X_{k+1} = \frac{1}{2}(a_kX_k - b_kX_k Y_k X_k) \\
Y_{k+1} = \frac{1}{2}(a_kY_k - b_kY_k X_k Y_k)
\end{cases}$ & $a_k, b_k \in [0, 5]$ \\
Visser & $X_{k+1} = a_kX_k + b_k(A - X_k^2)$ & $a_k, b_k \in [0, 10]$ \\
Visser\_Coupled & 
$\begin{cases}
X_{k+1} = a_kX_k + b_k(A - X_k^2) \\
Y_{k+1} = a_kY_k + b_k(I - X_kY_k)
\end{cases}$ & $a_k, b_k \in [0, 10]$ \\
Coupling & $Y_k = X_kA^{-1}$ & -- \\
\midrule
\multicolumn{3}{c}{\textbf{Actions for 1/3-th Root}} \\
\midrule
Iannazzo &
$\begin{cases}
X_{k+1} = X_k \left(\frac{a_kI + b_kY_k}{3}\right) \\
Y_{k+1} = \left(\frac{a_kI + b_kY_k}{3}\right)^{-3} Y_k
\end{cases}$ & $a_k, b_k \in [0, 10]$ \\
Newton & $X_{k+1} = \frac{1}{3}(a_kX_k + b_kX_k A^{-2})$ & $a_k, b_k \in [0, 10]$ \\
Visser & $X_{k+1} = a_kX_k + b_k(A - X_k^3)$ & $a_k, b_k \in [0, 10]$ \\
Coupling & $Y_k = AX_k^{-3}$ & -- \\
\bottomrule
\end{tabular}
\label{table:actionlist}
\end{table}

\subsection{List of Distributions}
The list of distributions we used throughout the experiments are as follows:\\
1. \textbf{Wishart} denotes $A = \frac{X^\top X}{3d} + \epsilon_{\text{stb}} I$ where $X \in \mathbb{R}^{d/4 \times d}$, $X_{ij} \sim \mathcal{N}(0, 1)$ i.i.d.. $\epsilon_{\text{stb}} = 1e-3$ exists for numerical stability.\\
2. \textbf{Uniform} denotes $A = QDQ^{T}$ where $Q$ is sampled from a Haar distribution and $D$ is a diagonal matrix where its entries are sampled from uniform $[-1, 1]$. We cap the diagonal entries with absolute value < 1e-3 to 1e-3.\\
3. \textbf{Hessian of Quartic} is the indefinite Hessian of a $d$-dimensional quartic $\sum_i z_i^4/4 - z_i^2/4$ evaluated at a random point $z \sim \mathcal{N}(0, \mathbf{I_d})$. We cap the eigenvalues with absolute value < 1e-3 to 1e-3, and normalize with Frobenius norm.\\
4. \textbf{CIFAR-10} is the random input matrix is $\hat\Sigma = \frac{1}{n} X^T X$  where $X \in \mathbb{R}^{n \times d}$ is a random batch of $n$ flattened CIFAR-10 images. We normalize with the Frobenius norm and add $\epsilon_{stb}I$ for $\epsilon_{stb} = 1e-3$.\\
5. \textbf{Erdos-Renyi} is the normalized graph Laplacian of a random Erdos-Renyi graph. We set $p = 0.4$ and $d = 5000$ for the experiments.
\clearpage

\section{List of all experimental results}
\label{sec:expresults_raw}

\subsection{Different matrix functions}

\textbf{Inverse} We learn to compute matrix inverse for two different distributions, Wishart and Uniform. Unfortunately, in our experiments, using Newtonschulz to compute matrix inverse was much slower than directly using torch.linalg.inv. However for uniform distribution, we had a more precise approximation of the inverse in terms of the loss than torch.linalg.inv.

\begin{figure}[H]
    \centering
    \begin{subfigure}[b]{0.45\textwidth}
        \centering
        \includegraphics[width=\linewidth]{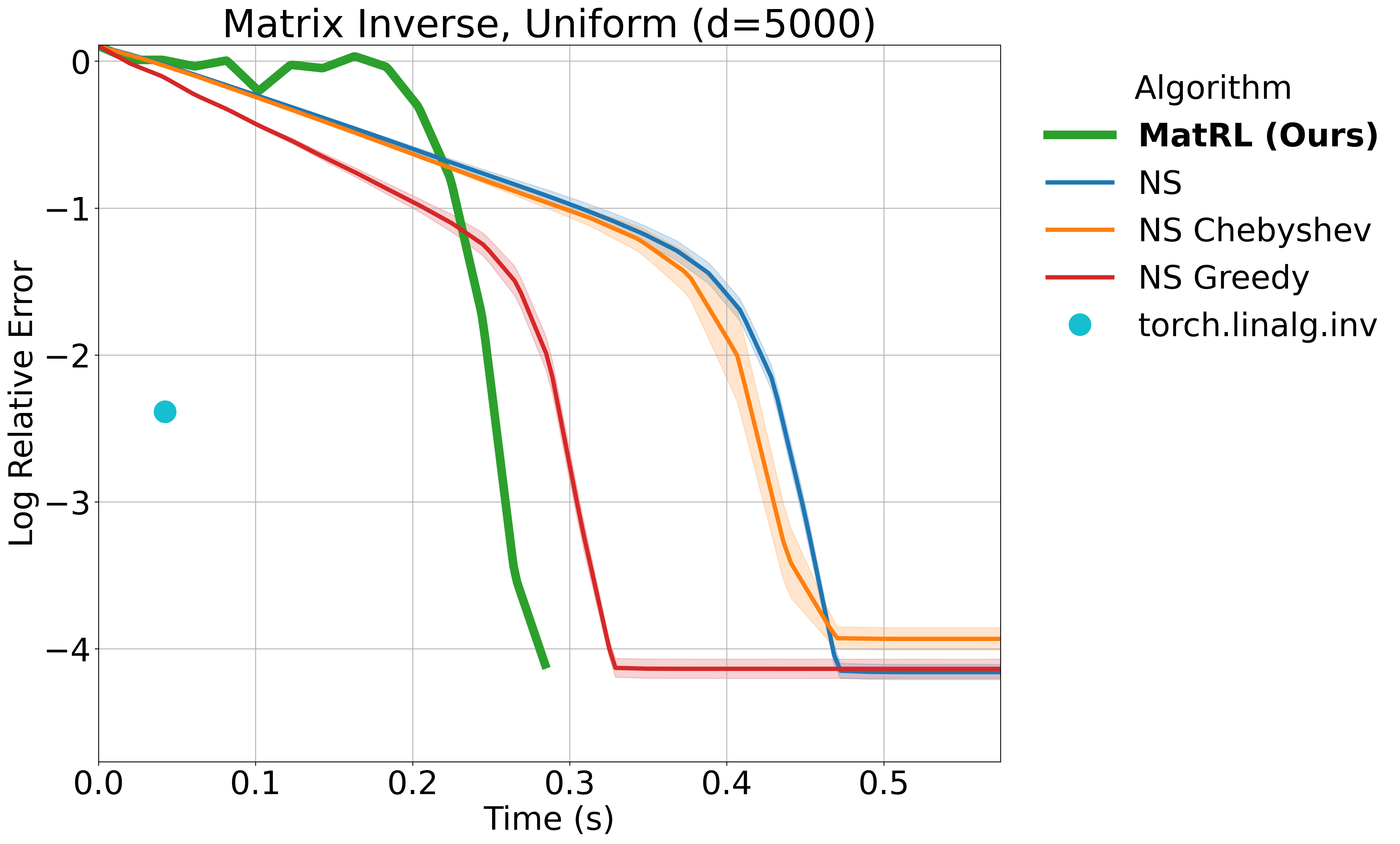}
        \caption{Uniform distribution}
    \end{subfigure}
    \hfill
    \begin{subfigure}[b]{0.45\textwidth}
        \centering
        \includegraphics[width=\linewidth]{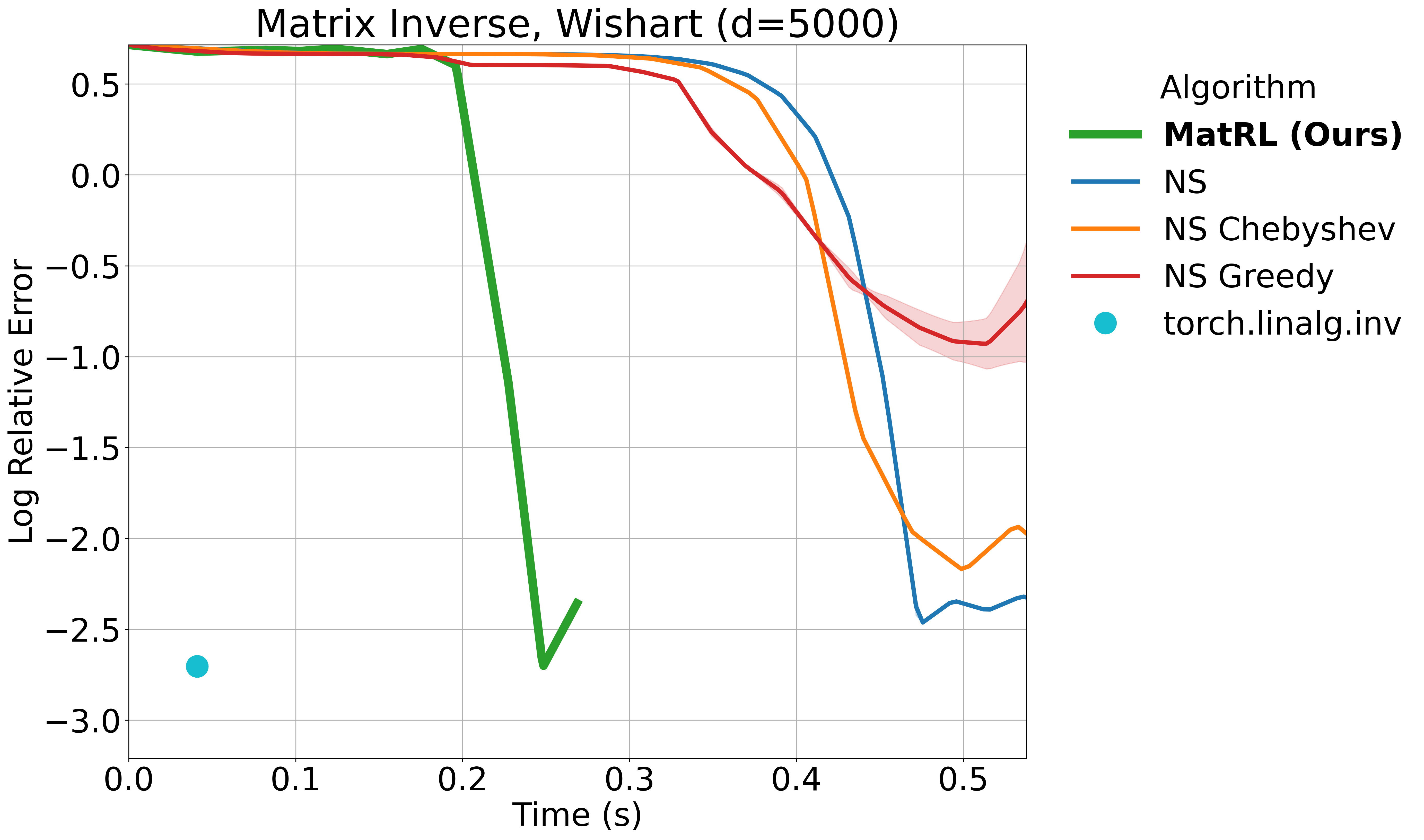}
        \caption{Wishart distribution}
    \end{subfigure}
    \caption{Computing matrix inverse with NewtonSchulz and variants}
\end{figure}

\textbf{Matrix sign} We learn matrix sign for Quartic Hessian and for matrices with Uniform [-1, 1] diagonal entries. Here $d = 5000$. For quartic hessian we use $\epsilon_{tol} = 1e-11$.

\begin{figure}[H]
    \centering
    \begin{subfigure}[b]{0.45\textwidth}
        \centering
        \includegraphics[width=\linewidth]{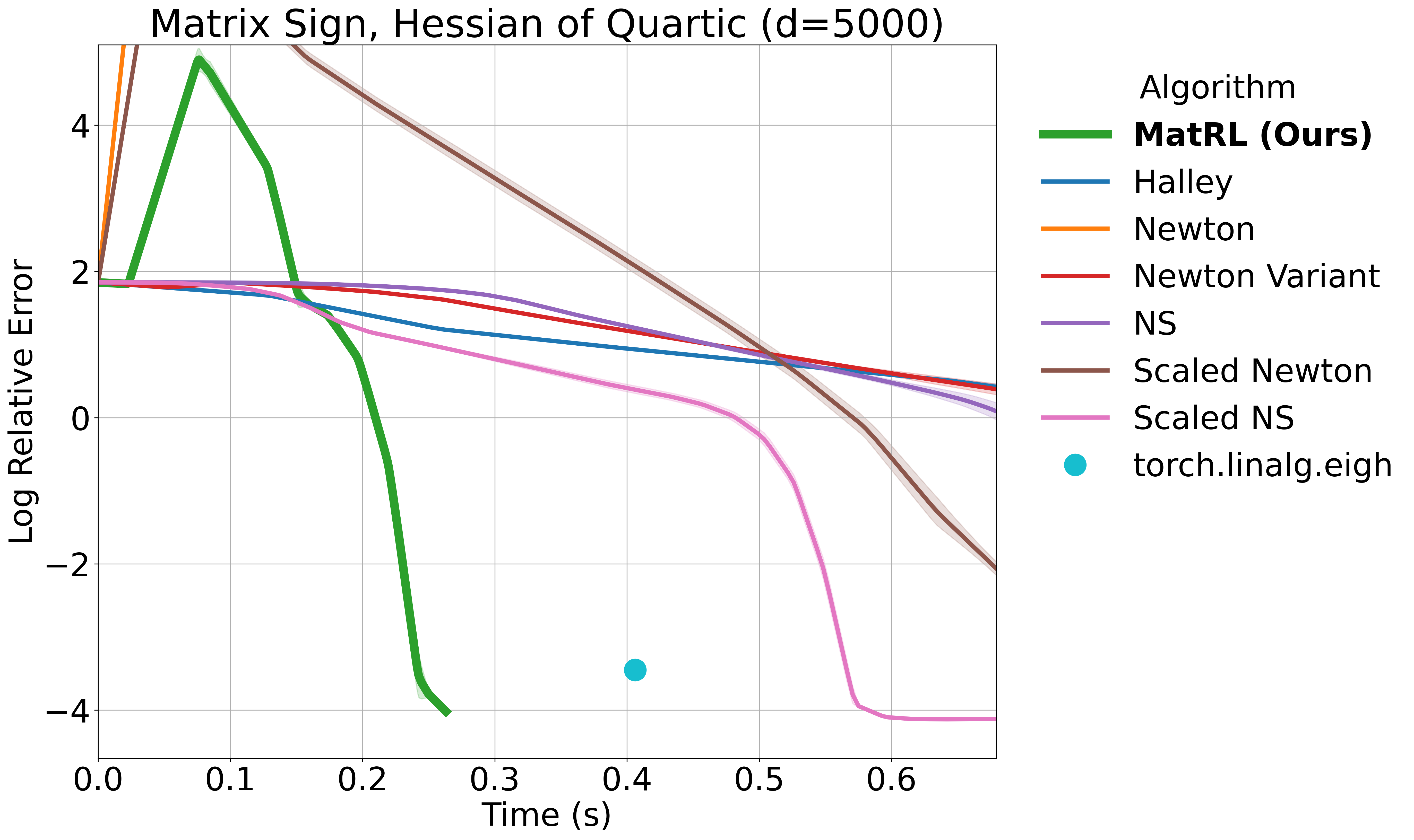}
        \caption{Hessian of Quartic}
    \end{subfigure}
    \hfill
    \begin{subfigure}[b]{0.45\textwidth}
        \centering
        \includegraphics[width=\linewidth]{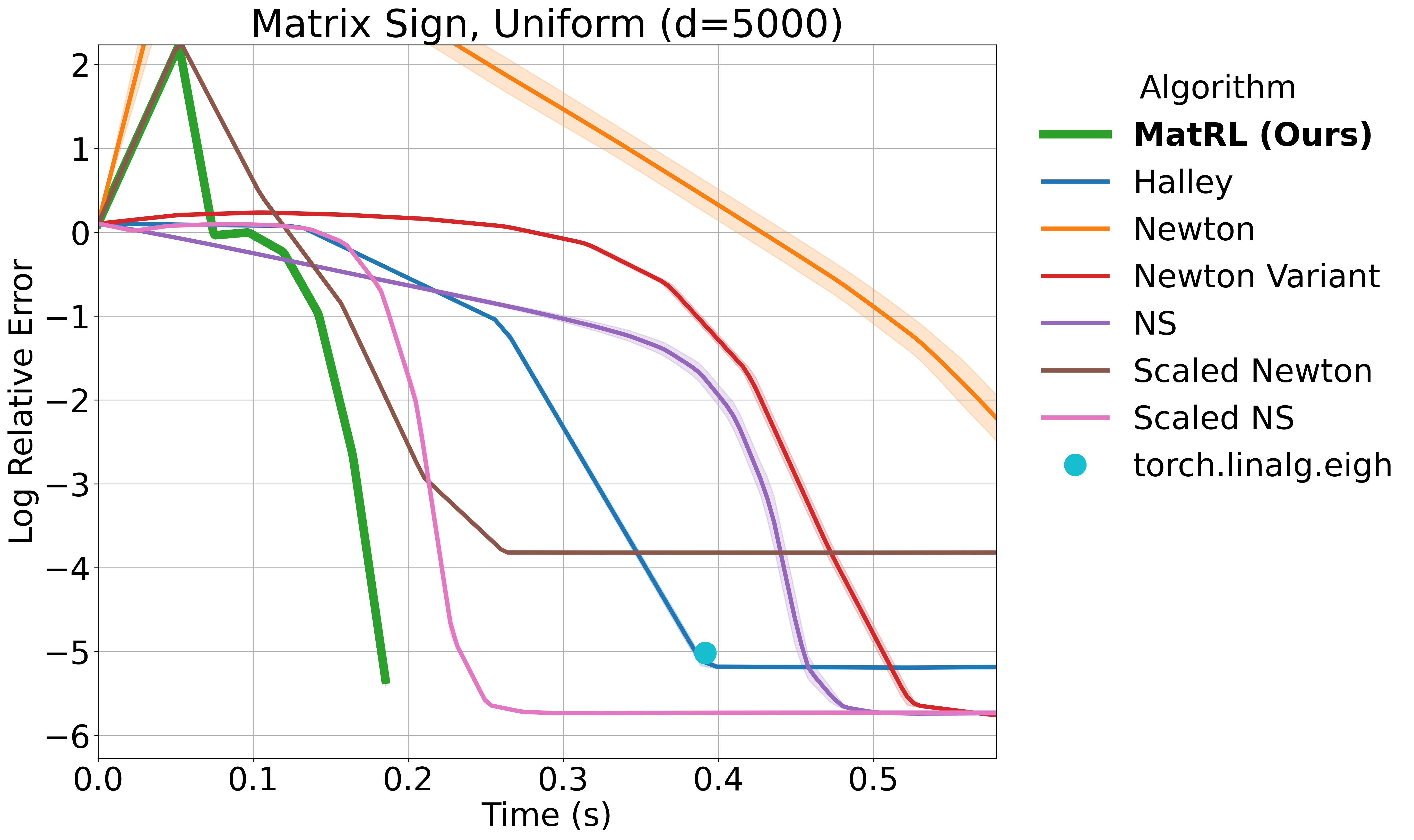}
        \caption{Uniform distribution}
    \end{subfigure}
    \caption{Computing matrix sign with NewtonSchulz and variants}
\end{figure}
\textbf{Matrix sqrt} We learn matrix sqrt for CIFAR-10 and Wishart matrices with $d=5000$. For CIFAR-10, double precision was used to learn the algorithm.

\begin{figure}[H]
    \centering
    \begin{subfigure}[b]{0.45\textwidth}
        \centering
        \includegraphics[width=\linewidth]{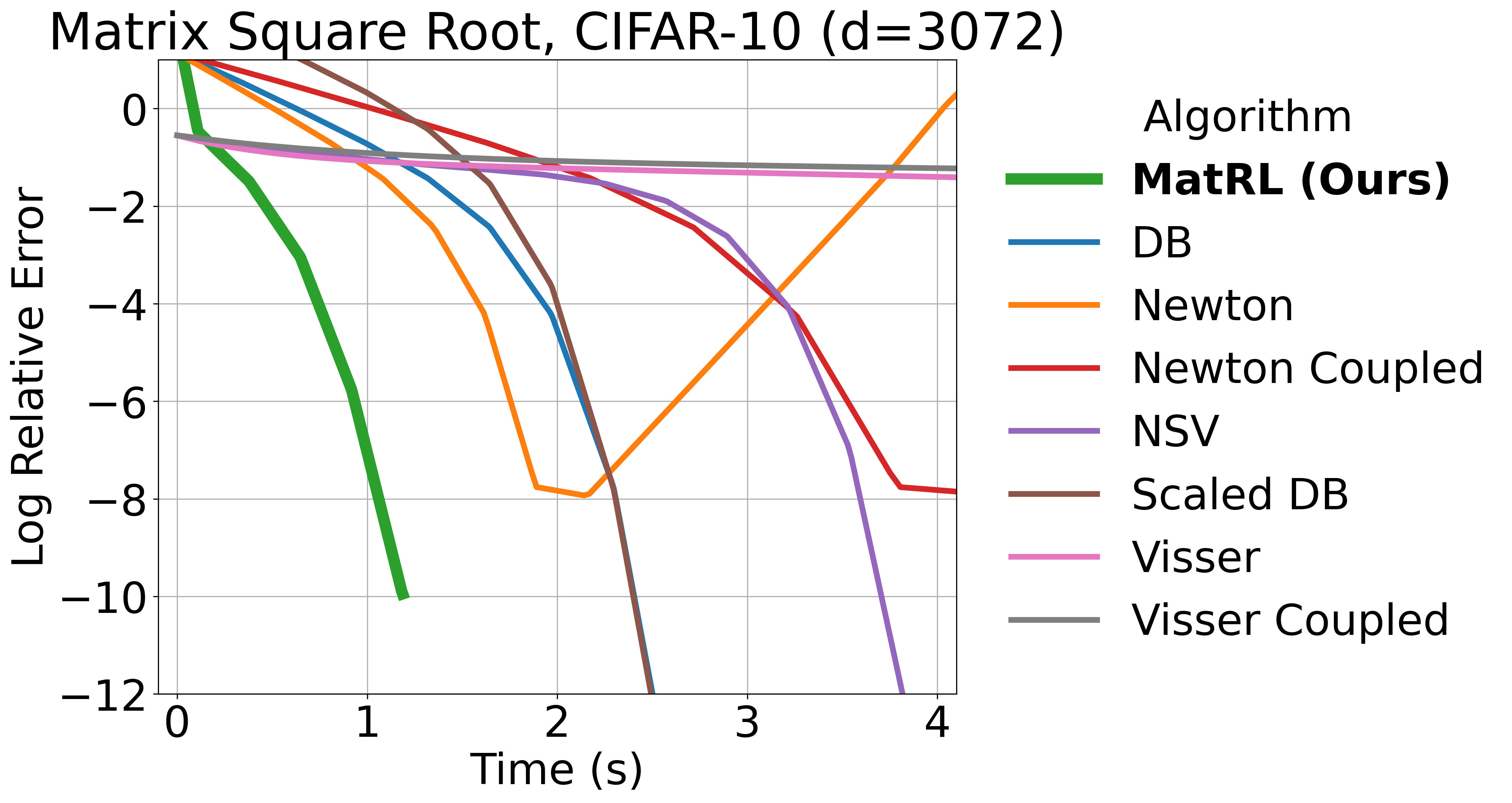}
        \caption{CIFAR-10}
    \end{subfigure}
    \hfill
    \begin{subfigure}[b]{0.45\textwidth}
        \centering
        \includegraphics[width=\linewidth]{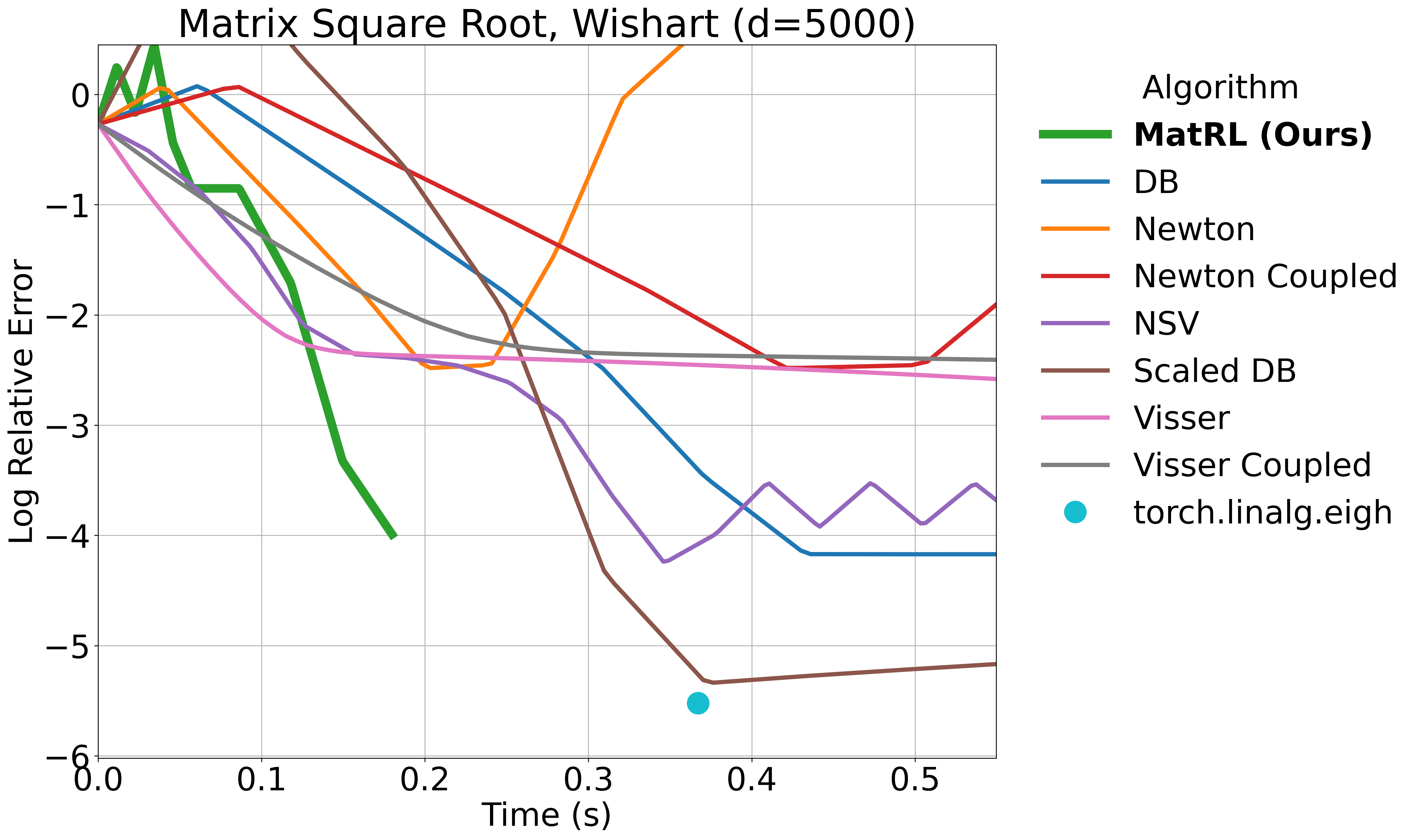}
        \caption{Wishart distribution}
    \end{subfigure}
    \caption{Computing matrix square root with NewtonSchulz and variants}
\end{figure}

\textbf{Matrix 1/3-root} We learn matrix 1/3-root for Wishart matrices and Erdos-Renyi graph. For Wishart matrices the method is not very effective: torch.linalg.eigh can find matrix 1/3-root with better accuracy with the same amount of time. However, for normalized graph Laplacians of Erdos-Renyi graph, it finds a faster algorithm with almost similar accuracy.

\begin{figure}[H]
    \centering
    \begin{subfigure}[b]{0.45\textwidth}
        \centering
        \includegraphics[width=\linewidth]{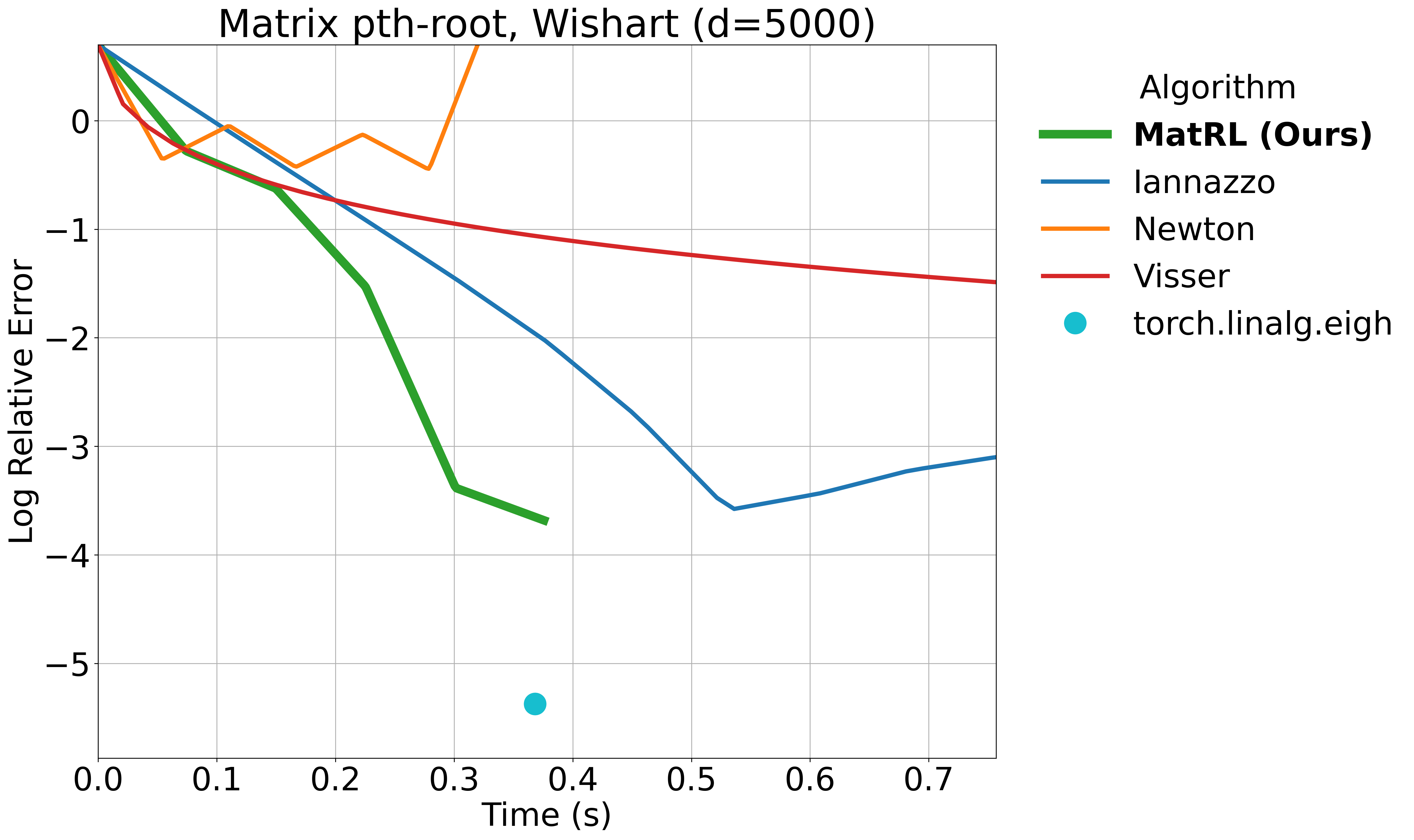}
        \caption{Wishart distribution}
    \end{subfigure}
    \hfill
    \begin{subfigure}[b]{0.45\textwidth}
        \centering
        \includegraphics[width=\linewidth]{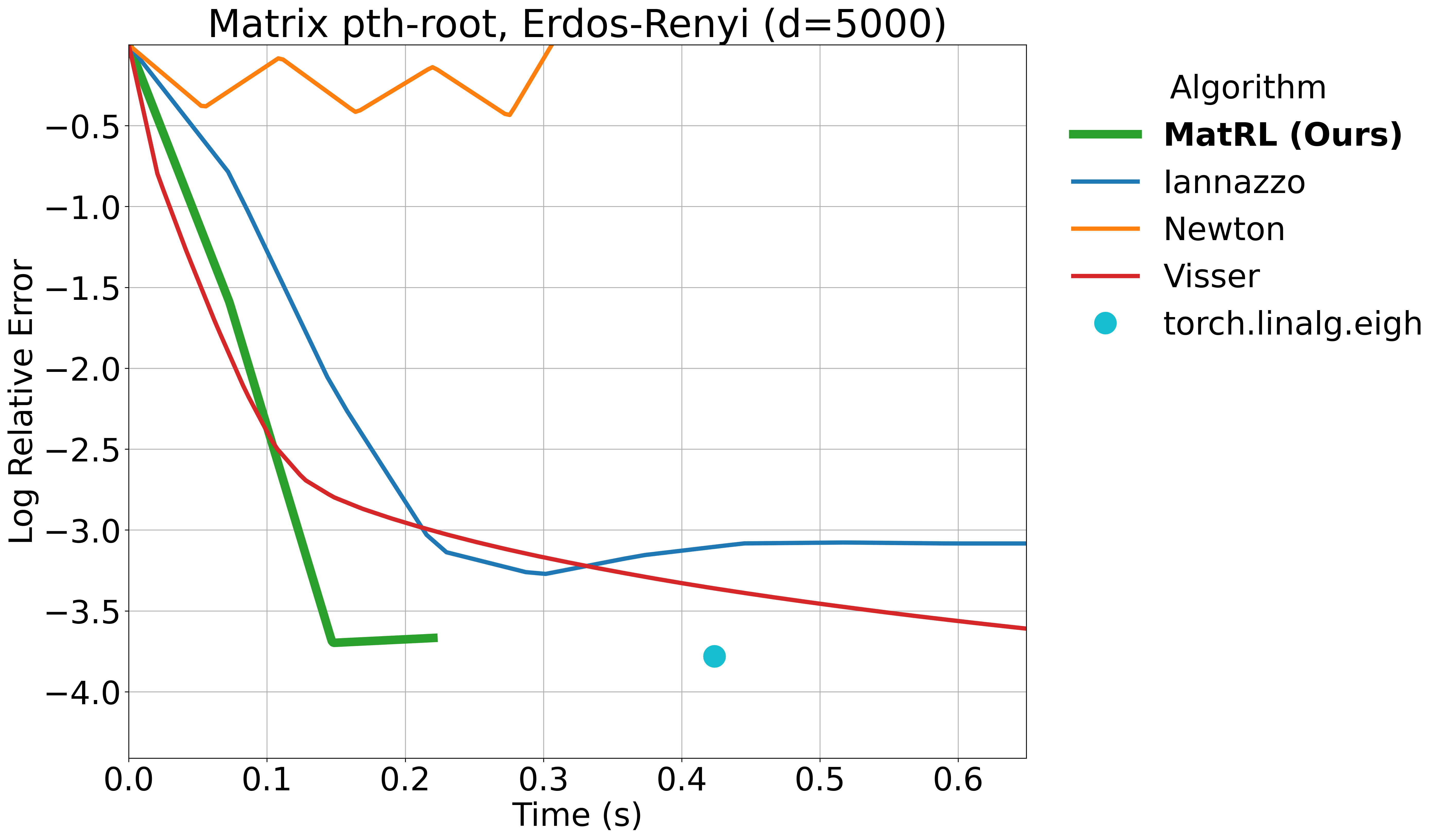}
        \caption{Normalized graph Laplacian of Erdős-Rényi}
    \end{subfigure}
    \caption{Examples of matrix distributions: structural or spectral views.}
\end{figure}

\subsection{\cref{MatRL} Adapting: sizes, precision, compute}

We demonstrate that different $d$, precision (float or double), and compute (GPU/CPU) can lead to different algorithms with matrix sign computation. We show both the loss curve and the found algorithm in each case.

\textbf{Different problem sizes} Here we show the results to compute matrix sign on random matrices with spectrum $Unif[-1, 1]$. $d = 1500, 3000, 5000, 10000$. One trend that we see is that for small $d$, we tend to use NewtonSchulz more, whereas for larger $d$ we tend to use Newton step more. It is related with the relative cost between Newton step and Newtonschulz step.

\begin{figure}[H]
    \centering
    \begin{subfigure}[b]{0.45\textwidth}
        \centering
        \includegraphics[width=\linewidth]{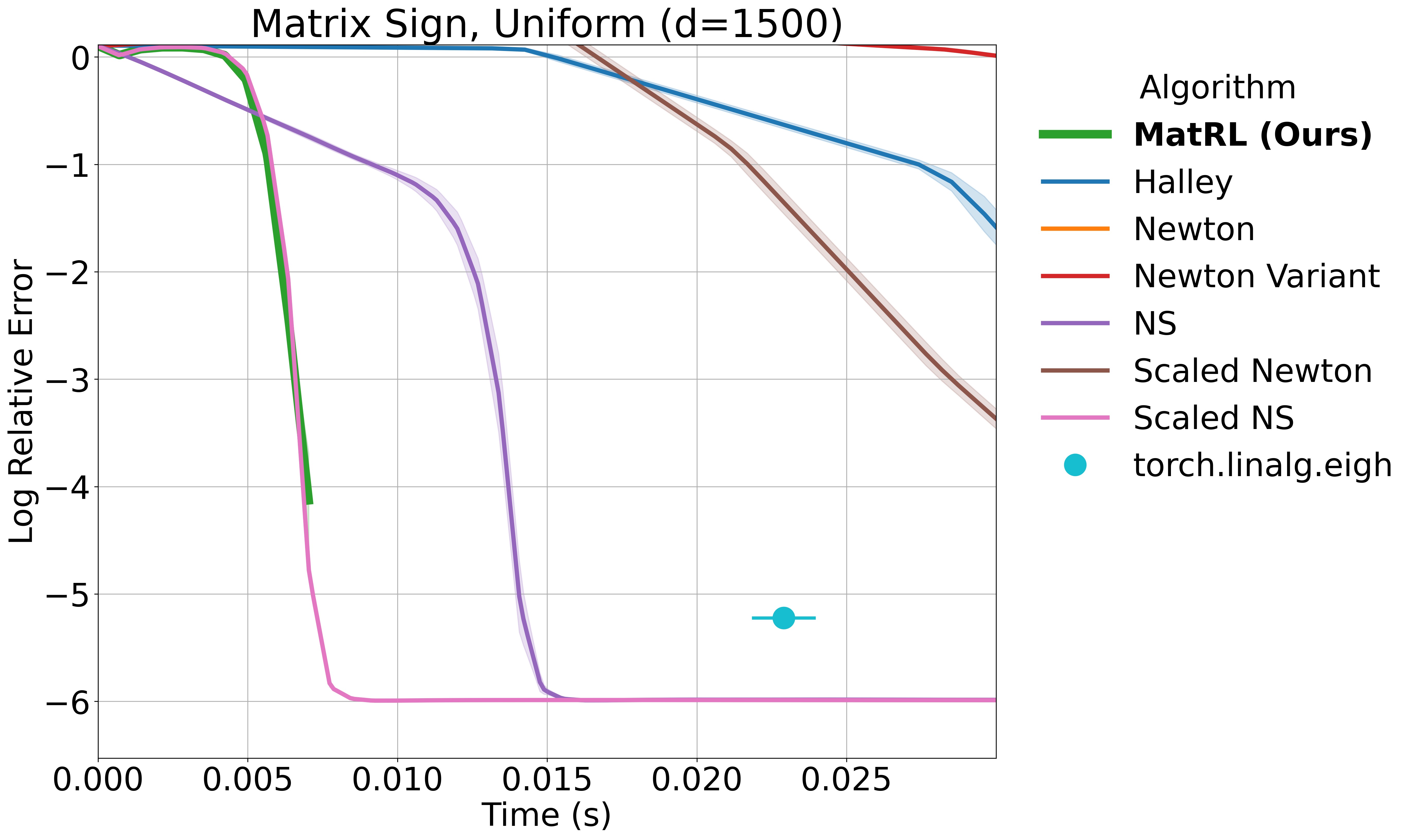}
        \caption{$d = 1500$}
    \end{subfigure}
    \hfill
    \begin{subfigure}[b]{0.45\textwidth}
        \centering
        \includegraphics[width=\linewidth]{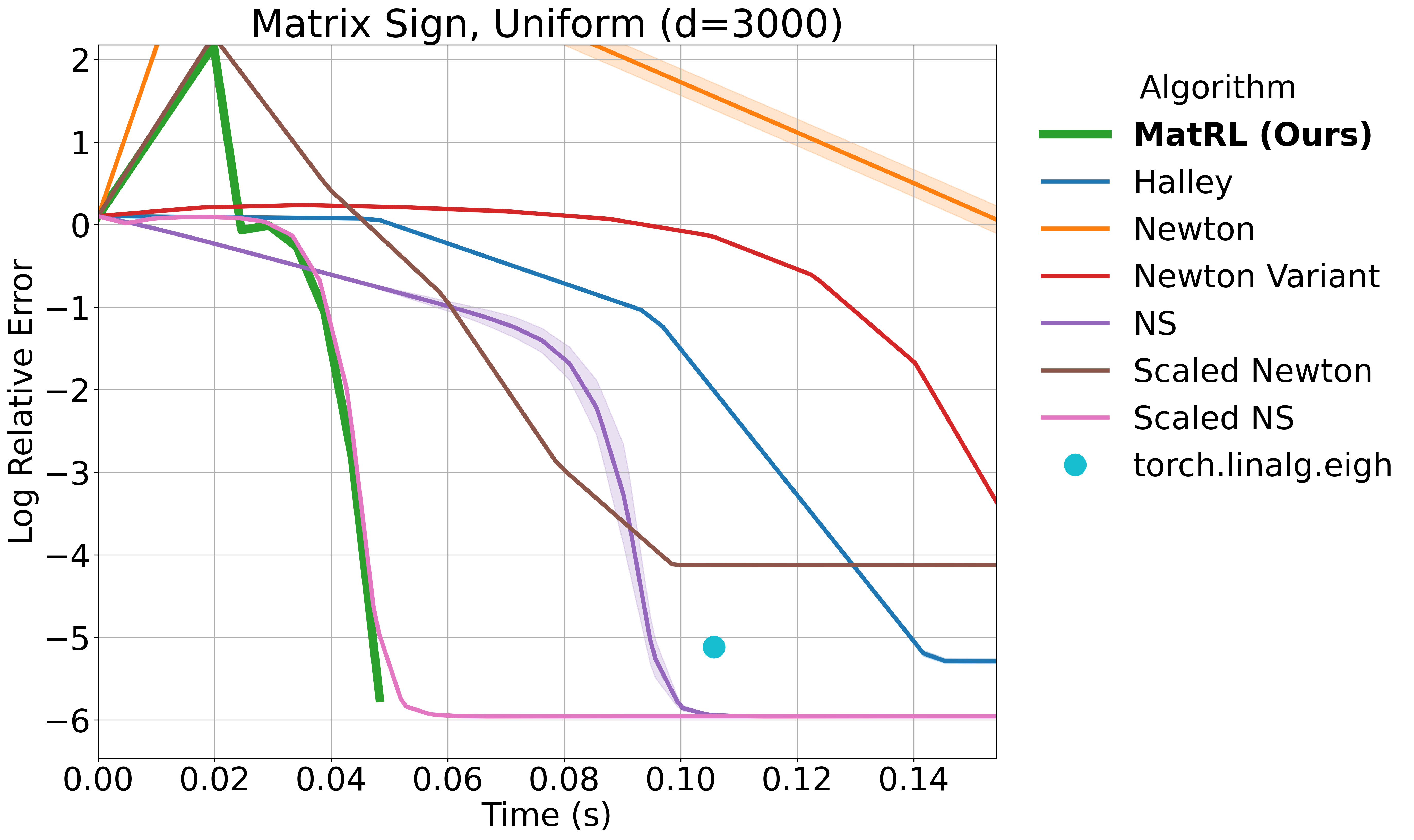}
        \caption{$d = 3000$}
    \end{subfigure}
    \begin{subfigure}[b]{0.45\textwidth}
        \centering
        \includegraphics[width=\linewidth]{Figures/sign_unif_d=5000.png}
        \caption{$d = 5000$}
    \end{subfigure}
    \hfill
    \begin{subfigure}[b]{0.45\textwidth}
        \centering
        \includegraphics[width=\linewidth]{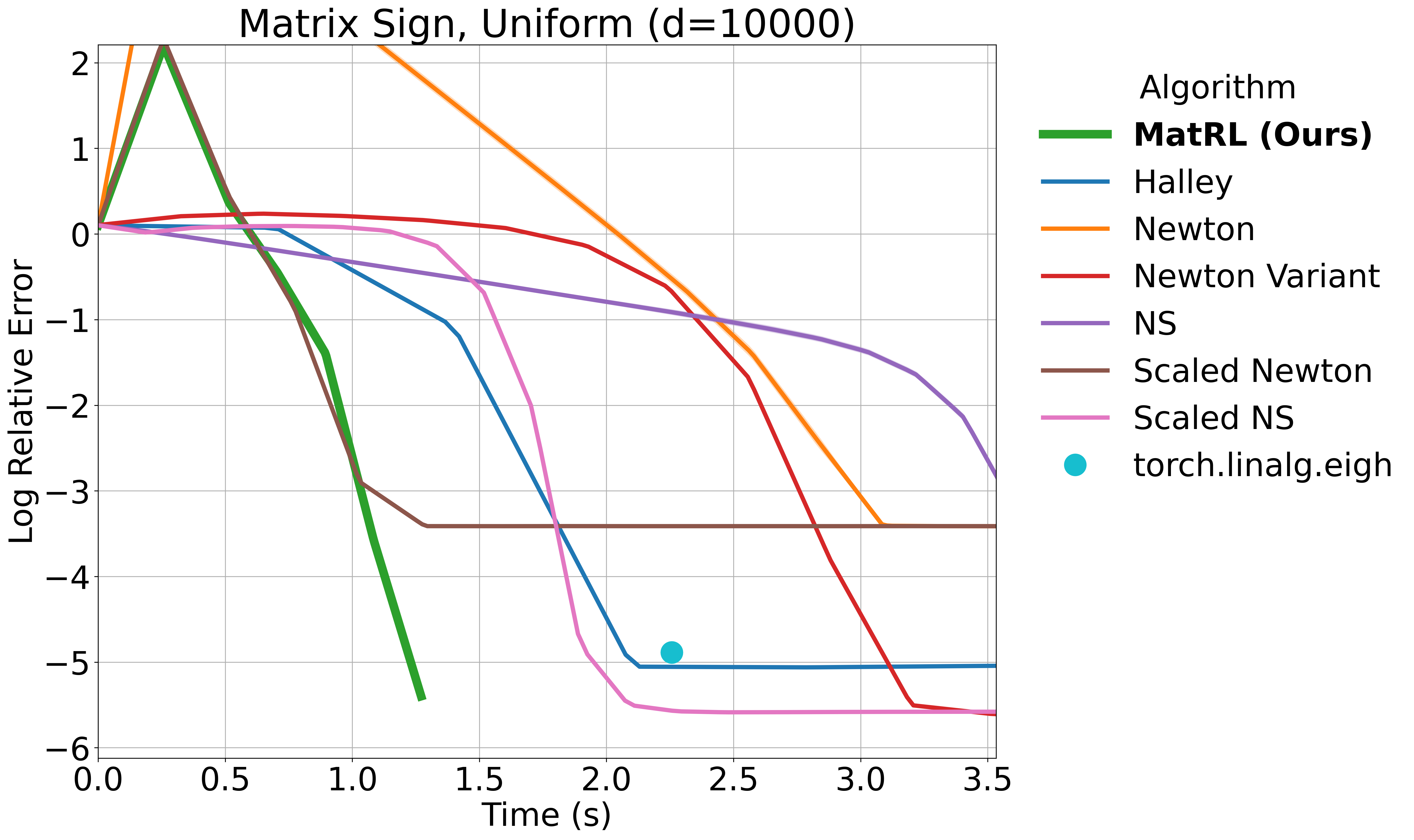}
        \caption{$d = 10000$}
    \end{subfigure}
    \caption{Computing matrix sign for different matrix sizes}
\end{figure}

The found algorithms for $d = 1500, 3000, 5000, 10000$ are as follows. We set $\epsilon_{tol} = 1e-11$ for $d=10000$.
\begin{algorithm}
\caption{Iterative SIGN for Uniform on GPU with $d = 1500$}
\begin{algorithmic}
  \STATE \textbf{Input:} $A$
  \STATE Initialize $X_0 = A$
  \STATE Set $a \gets [1.731, 1.729, 1.724, 1.712, 1.680, 1.606, 1.439, 1.190, 1.029, 1.000]$, 
  \STATE // rounded off to three digits
  \FOR {$i=1$ to 10}
  \STATE $X_i = a_{i-1}X_{i-1} + 0.5(a_{i-1}X_{i-1} - (a_{i-1}X_{i-1})^3)$
  \ENDFOR
  \RETURN $X_9$
\end{algorithmic}
\end{algorithm}
\begin{algorithm}
\caption{Iterative SIGN for Uniform on GPU with $d = 3000$}
\begin{algorithmic}
  \STATE \textbf{Input:} $A$
  \STATE Initialize $X_0 = A$
  \STATE Set $a \gets [27.685], b \gets [0.086, 1.471, 1.411, 1.162, 1.021, 1.000]$, \\$c \gets [0.975, 0.5, 0.5, 0.5, 0.5, 0.5]$, 
  \STATE // rounded off to three digits
  \FOR {$i=1$ to 1}
  \STATE $X_i = 0.5(a_{i-1}X_{i-1}+(a_{i-1}X_{i-1})^{-1})$
  \ENDFOR
  \FOR {$i=2$ to 7}
  \STATE $X_i = b_{i-1}X_{i-1} + c_{i-1}(b_{i-1}X_{i-1} - (b_{i-1}X_{i-1})^3)$
  \ENDFOR
  \RETURN $X_7$
\end{algorithmic}
\end{algorithm}
\begin{algorithm}
\caption{Iterative SIGN for Uniform on GPU with $d = 5000$}
\begin{algorithmic}
  \STATE \textbf{Input:} $A$
  \STATE Initialize $X_0 = A$
  \STATE Set $a \gets [29.628], b \gets [0.099, 1.600, 1.427, 1.178, 1.025, 1.000]$, \\$c \gets [0.5, 0.5, 0.5, 0.5, 0.5, 0.5]$, 
  \STATE // rounded off to three digits
  \FOR {$i=1$ to 1}
  \STATE $X_i = 0.5(a_{i-1}X_{i-1}+(a_{i-1}X_{i-1})^{-1})$
  \ENDFOR
  \FOR {$i=2$ to 7}
  \STATE $X_i = b_{i-1}X_{i-1} + c_{i-1}(b_{i-1}X_{i-1} - (b_{i-1}X_{i-1})^3)$
  \ENDFOR
  \RETURN $X_7$
\end{algorithmic}
\end{algorithm}
\begin{algorithm}
\caption{Iterative SIGN for Uniform on GPU with $d = 10000$}
\begin{algorithmic}
  \STATE \textbf{Input:} $A$
  \STATE Initialize $X_0 = A$
  \STATE Set $a \gets [28.790, 0.239], b \gets [0.609, 1.107,1.009, 1]$, 
  \STATE // rounded off to three digits
  \FOR {$i=1$ to 2}
  \STATE $X_i = 0.5(a_{i-1}X_{i-1}+(a_{i-1}X_{i-1})^{-1})$
  \ENDFOR
  \FOR {$i=3$ to 6}
  \STATE $X_i = b_{i-1}X_{i-1} + 0.5(b_{i-1}X_{i-1} - (b_{i-1}X_{i-1})^3)$
  \ENDFOR
  \RETURN $X_6$
\end{algorithmic}
\end{algorithm}

\newpage
\textbf{Different precision} Here we show the results to compute matrix sign on random matrices with spectrum $Unif[-1, 1]$ for float and double precision. For double precision when $d = 5000$, NewtonSchulz becomes as expensive as Newton step whereas Newton step is more effective - hence we use Newton until the end. For float precision the method finds a mixture of Newton and NewtonSchulz. For double we used $\epsilon_{tol} = 1e-11$.

\begin{figure}[H]
    \centering
    \begin{subfigure}[b]{0.45\textwidth}
        \centering
        \includegraphics[width=\linewidth]{Figures/sign_unif_d=5000.png}
        \caption{Float}
    \end{subfigure}
    \hfill
    \begin{subfigure}[b]{0.45\textwidth}
        \centering
        \includegraphics[width=\linewidth]{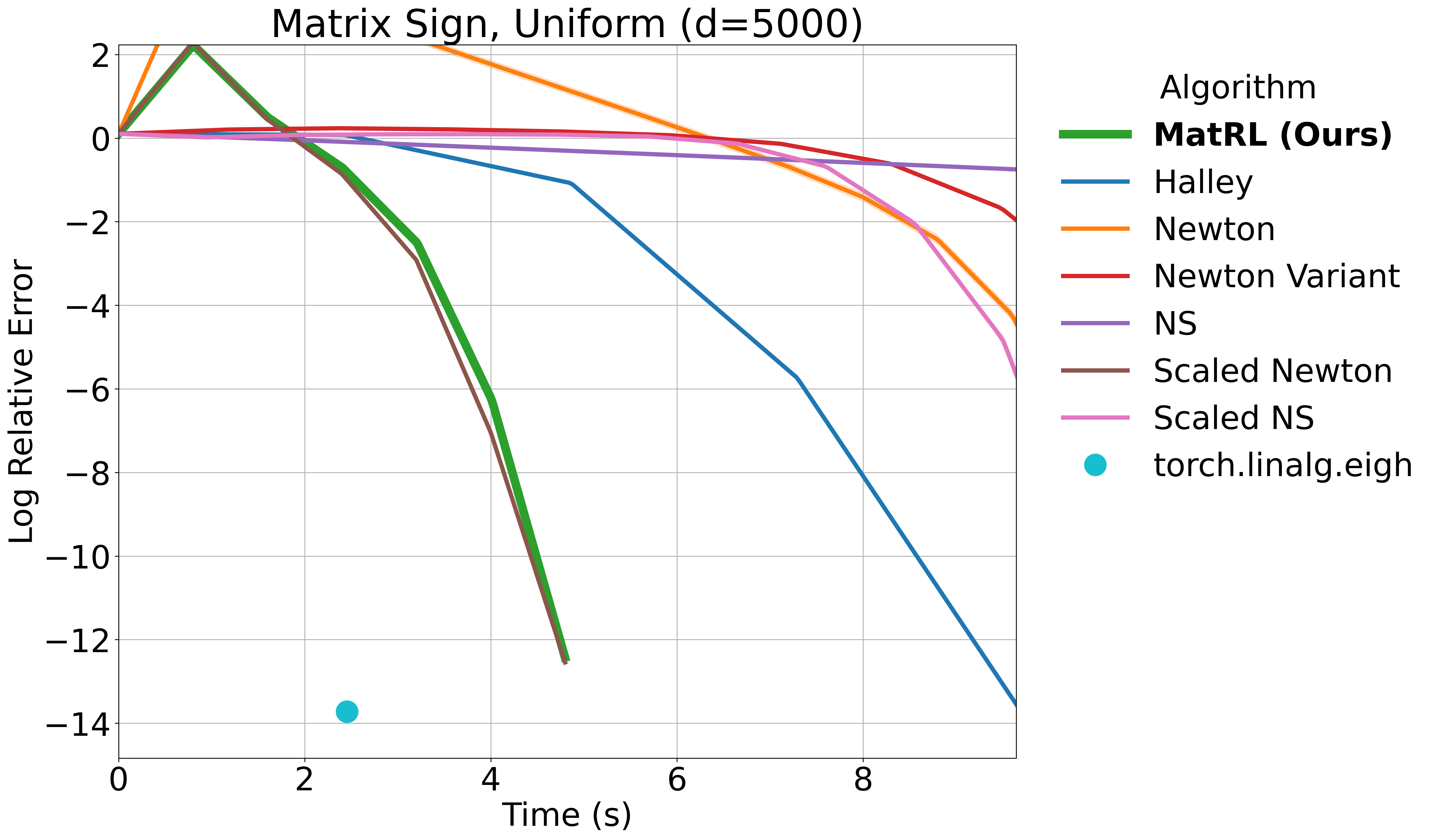}
        \caption{Double}
    \end{subfigure}
    \caption{Computing matrix sign for different precision}
\end{figure}

\begin{algorithm}
\caption{Iterative SIGN for Uniform on GPU with $d = 5000$, DOUBLE}
\begin{algorithmic}
  \STATE \textbf{Input:} $A$
  \STATE Initialize $X_0 = A$
  \STATE Set $a \gets [29.459, 0.290, 0.624, 0.947,0.999,0.999]$,
  \STATE // rounded off to three digits
  \FOR {$i=1$ to 6}
  \STATE $X_i = 0.5(a_{i-1}X_{i-1}+(a_{i-1}X_{i-1})^{-1})$
  \ENDFOR
  \RETURN $X_6$
\end{algorithmic}
\end{algorithm}

\textbf{Different compute} We also run MatRL on GPU and on CPU. The difference that occurs here is also similar in vein: on a GPU, Newton step is $\approx$ x2.28 more costly than a Newtonschulz step, whereas on a CPU it is $\approx$ x1.62 more costly. This makes the algorithm found on CPU use Newton step more. For CPU we also used $\epsilon_{tol} = 1e-11$ for better convergence.

\begin{figure}[H]
    \centering
    \begin{subfigure}[b]{0.45\textwidth}
        \centering
        \includegraphics[width=\linewidth]{Figures/sign_unif_d=5000.png}
        \caption{GPU (NVIDIA RTX A-6000)}
    \end{subfigure}
    \hfill
    \begin{subfigure}[b]{0.45\textwidth}
        \centering
        \includegraphics[width=\linewidth]{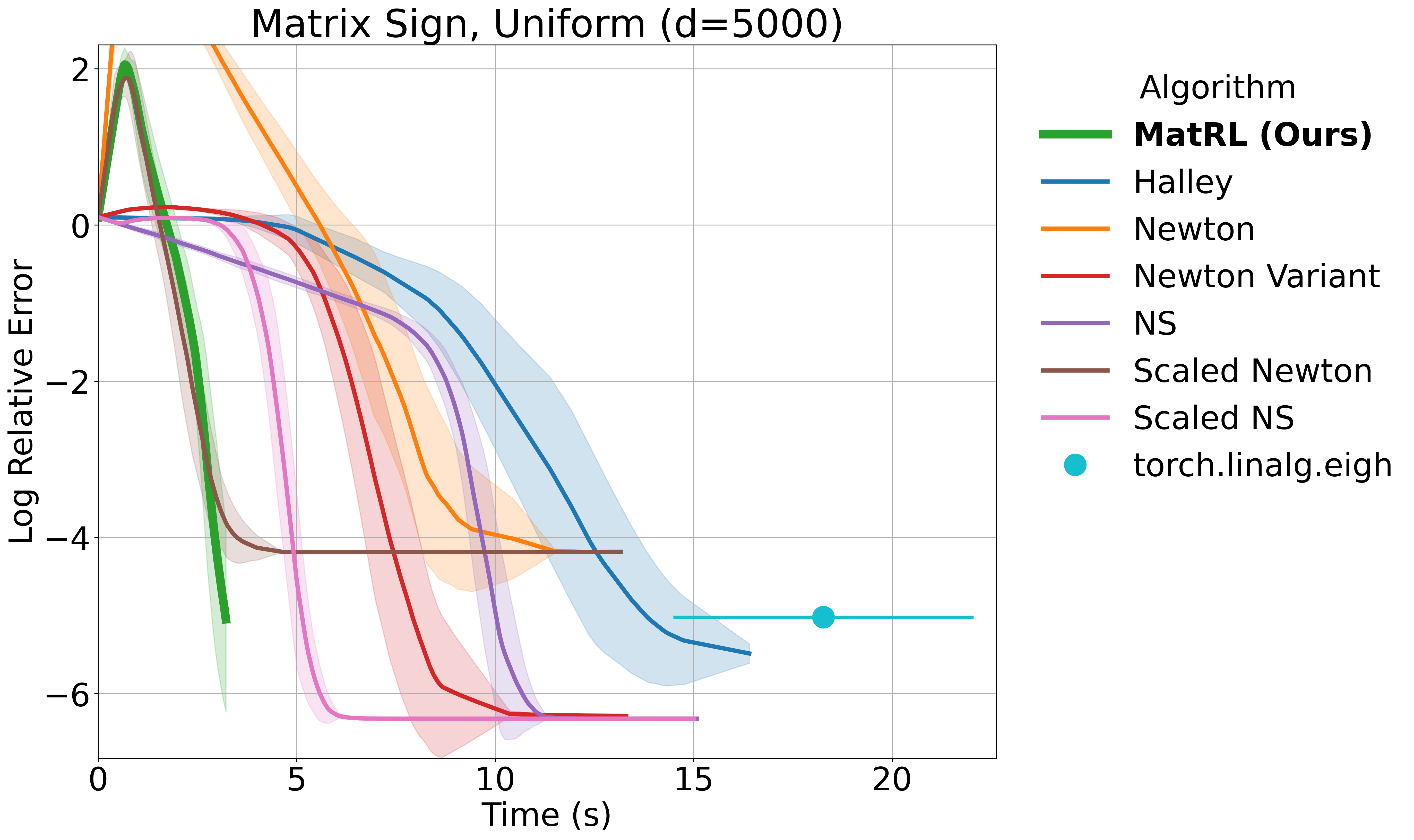}
        \caption{CPU (AMD EPYC 7713 64-Core Processor)}
    \end{subfigure}
    \caption{Computing matrix sign on different machines}
\end{figure}

\begin{algorithm}
\caption{Iterative SIGN for Uniform on CPU with $d = 5000$}
\begin{algorithmic}
  \STATE \textbf{Input:} $A$
  \STATE Initialize $X_0 = A$
  \STATE Set $a \gets [32.273, 0.255, 0.676], b \gets [0.962, 1.001,1.000]$, 
  \STATE // rounded off to three digits
  \FOR {$i=1$ to 3}
  \STATE $X_i = 0.5(a_{i-1}X_{i-1}+(a_{i-1}X_{i-1})^{-1})$
  \ENDFOR
  \FOR {$i=4$ to 6}
  \STATE $X_i = b_{i-1}X_{i-1} + 0.5(b_{i-1}X_{i-1} - (b_{i-1}X_{i-1})^3)$
  \ENDFOR
  \RETURN $X_6$
\end{algorithmic}
\end{algorithm}

\end{document}